\documentclass{article}
\usepackage[pass]{geometry}

\usepackage{microtype}

\usepackage{times}

\usepackage{graphicx}
\usepackage{subfigure}
\usepackage{booktabs} 
\usepackage{thmtools}
\usepackage{thm-restate}

\usepackage{amsmath}
\usepackage{amsfonts, dsfont} 
\usepackage{hyperref}       
\usepackage{url}            
\usepackage{amsthm}
\usepackage{listings}
\usepackage{float}

\DeclareMathOperator*{\argmax}{arg\,max}
\DeclareMathOperator*{\argmin}{arg\,min}
\DeclareMathOperator{\median}{median}
\DeclareMathOperator{\len}{len}
\DeclareMathOperator{\MAD}{MAD}

\DeclareMathOperator{\AD}{AD}
\DeclareMathOperator{\KL}{KL}
\DeclareMathOperator{\NEEDY}{NEEDY}
\DeclareMathOperator{\findad}{FindAD}

\usepackage[usenames,dvipsnames]{xcolor}

\graphicspath{{figure/}}
\newtheorem{theorem}{Theorem}
\newtheorem{lemma}{Lemma}

\newtheorem{definition}{Definition}

\newtheorem{remark}{Remark}

\usepackage{hyperref}

\usepackage[accepted]{icml2020}
\newcommand{\roaielim}{{\tt ROAIElim }}
\newcommand{\roailucb}{{\tt ROAILUCB }}
\newcommand{\roai}{{\tt ROAI}}
\usepackage[capitalize]{cleveref}

\icmltitlerunning{Robust Outlier Arm Identification}

\begin{document}
	
	\twocolumn[
	\icmltitle{Robust Outlier Arm Identification}
	
	
	
	
	\begin{icmlauthorlist}
		\icmlauthor{Yinglun Zhu}{uw}
		\icmlauthor{Sumeet Katariya}{am}
		\icmlauthor{Robert Nowak}{uw}
	\end{icmlauthorlist}
	
	\icmlaffiliation{uw}{University of Wisconsin-Madison}
	\icmlaffiliation{am}{Amazon}
	\icmlcorrespondingauthor{Yinglun Zhu}{yinglun@cs.wisc.edu}
	
	\icmlkeywords{Machine Learning, ICML}
	
	\vskip 0.3in
	]
	
	
	
	\printAffiliationsAndNotice{}  
	
	\begin{abstract}
        
        We study the problem of Robust Outlier Arm Identification (ROAI), where the goal is to identify arms whose expected rewards deviate substantially from the majority, by adaptively sampling from their reward distributions. We compute the outlier threshold using the median and median absolute deviation of the expected rewards. This is a robust choice for the threshold compared to using the mean and standard deviation, since it can identify outlier arms even in the presence of extreme outlier values. Our setting is different from existing pure exploration problems where the threshold is pre-specified as a given value or rank. This is useful in applications where the goal is to identify the set of promising items but the cardinality of this set is unknown, such as finding promising drugs for a new disease or identifying items favored by a population. We propose two $\delta$-PAC algorithms for ROAI, which includes the first UCB-style algorithm for outlier detection, and derive upper bounds on their sample complexity. We also prove a matching, up to logarithmic factors, worst case lower bound for the problem, indicating that our upper bounds are generally unimprovable. Experimental results show that our algorithms are both robust and about $5$x sample efficient compared to state-of-the-art.
        

	\end{abstract}

	\section{Introduction}

	Multi-armed bandits are commonly used to identify the optimal items (arms) among multiple candidates through adaptive queries (pure exploration setting \citep{jamieson2014best}). Every item is associated with an unknown probability distribution, and when a bandit algorithm selects (pulls) an item, it observes a value (reward) sampled from this distribution. Depending on its objective and the history of observed values, the bandit algorithm has to decide which item to sample at every time $t$, so as to identify the optimal items using as few samples as possible. Pure exploration bandit algorithms have been proposed for various objectives, such as identifying arms with the largest rewards \citep{jamieson2014lil, jamieson2014best, chen2016towards}, identifying arms above a given threshold \citep{locatelli2016optimal, mukherjee2017thresholding, xu2019thresholding} or clustering arms \citep{katariya2018adaptive, katariya2019maxgap}. 
	
	In this paper, we study bandit algorithms for identifying outlier arms. Outlier arms are defined as those with expected rewards that are outliers relative to the overall set of expected rewards (e.g., arms with expected rewards that are several deviations above the mean/median of the overall set of expected rewards). The outlier detection problem has wide applications in scientific discovery \citep{grun2015single, chaudhary2015folding}, fraud detection \citep{porwal2018credit}, medicine \citep{schiff2017screening}, and public health \citep{hauskrecht2013outlier}. In contrast to passive outlier detection algorithms which identify outlier items using a pre-sampled dataset, bandit algorithms actively query items with the goal of identifying outliers using as few samples as possible. This is important because it can lead to early detection of fraud for example. Outlier arms subsume \emph{good} arms with expected rewards substantially above the average, and most applications mentioned in good arm identification \cite{kano2019good} apply to our setting.
	
	As observed in \citet{zhuang2017identifying}, bandit outlier detection cannot be reduced to best arm(s) identification in bandits because of the inherent double exploration dilemma - the threshold is unknown and any algorithm must balance exploring individual arms and exploring the outlier threshold. \citet{zhuang2017identifying} define the outlier threshold $\bar{\theta}$ using the k-sigma rule applied to the mean $\bar{\mu}$ and standard deviation $\bar{\sigma}$ of the expected rewards i.e., $\bar{\theta} = \bar{\mu} + k \cdot \bar{\sigma}$. However this threshold can fail to identify the correct outlier arms because the mean and standard deviation are themselves sensitive to outlier values (non-robust estimators). It can also miss outliers when the number of arms is small. In this paper, we define the threshold using the $k$-sigma rule applied to the \emph{median} and the \emph{median absolute deviation}, which are robust estimators with the highest possible breakdown point $0.5$. This is the recommended practice in literature \citep{hampel1974influence, huber2004robust, swallow1996using}, and emphasized by \citet{leys2013detecting} in their aptly titled paper: ``Detecting outliers: Do not use standard deviation around the mean, use absolute deviation around the median''. Similarly, \citet{chung2008median} conduct extensive experiments to compare the two methods and show that the median-based threshold identifies outliers that were missed by the mean-based threshold. We show through our theoretical and empirical results that this robust threshold not only identifies outliers more accurately, but it also requires fewer samples to do so than the mean-based threshold.

	\subsection{Contributions and Paper Organization}
	We make the following contributions. In Section \ref{sec_setting} we formally define the Robust Outlier Arm Identification (ROAI) problem with justifications from Huber's $\epsilon$-contamination model. In Section \ref{sec_algorithm}, we propose two algorithms for the ROAI problem, which includes the first UCB-style algorithm for outlier detection. We theoretically prove the correctness our algorithms and derive their sample complexity upper bounds in Section \ref{sec_analysis}. A matching, up to logarithmic factors, worst case lower bound is provided in Section \ref{sec_lower_bound}, indicating our upper bounds are generally tight. We further generalize our algorithms to settings with known contamination upper bound in Section \ref{sec_generalization}. Experiments conducted in Section \ref{sec_experiment} show that our algorithms are \emph{both} more robust and more sample efficient than previous state-of-the-art. We conclude our paper in Section \ref{sec_conclusion} with open problems. All proofs are deferred to the Appendix due to lack of space.
	

	\subsection{Related Work}
	\label{sec:related work}
	The pure exploration problem in the multi-armed bandit setting has a long history, starting from the work of \cite{bechhofer1958sequential, paulson1964sequential}.  The aim of pure exploration is to identify an arm or arms with certain properties. For example, the best-arm identification problem involves correctly deciding which arm has the largest expected reward. The instance-dependent sample complexity bound on the best arm identification problem was analyzed/improved by \cite{even2002pac, even2006action, gabillon2012best, karnin2013almost, jamieson2014lil, jamieson2014best, chen2016towards}. The problem was also generalized to the setting of identifying the top-$m$ arm \cite{kalyanakrishnan2012pac, chen2017adaptive, chen2017nearly}; the thresholding bandit \cite{locatelli2016optimal, mukherjee2017thresholding, xu2019thresholding} which identifies all arms with expected reward above a given threshold $\theta$; and the good arm identification problem \cite{kano2019good, katz2019true}, where for a given $\epsilon$ ``good arms'' have expected reward within $\epsilon$ of the largest. Lower bounds developed in the pure exploration setting \cite{mannor2004sample, chen2015optimal, kaufmann2016complexity, garivier2016optimal, simchowitz2017simulator} shed light on the optimality of existing algorithms.
	
	In all of the above settings, the subset of arms of interest is determined by a user-defined parameter, e.g., $m$, $\theta$,  and $\epsilon$. Outlier arm identification cannot be cast in these settings, since the cut-off cannot be a prespecified threshold or rank.  The cut-off depends on the overall distribution of expected rewards, which is unknown in advance. In other words, outlier arm identification has an \emph{instance-dependent identification target}. 
	Bandit problems with instance-dependent identification targets have attracted some attention recently. One of the work \cite{katariya2019maxgap} studies the problem of identifying the largest gap in the ordering of the expected rewards, which provides a natural separation of the arms into two groups or clusters. Another line of work \cite{zhuang2017identifying} focuses on identifying outlier arms with an outlier threshold adaptive to the bandit instance. Specifically, they use the threshold $\bar{\theta} = \bar{\mu} + k \cdot \bar{\sigma}$, with $\bar{\mu} $ and $\bar{\sigma}$ being the mean and standard deviation of distribution of expected rewards, respectively. The parameter $k$ is usually chosen as $2$ or $3$ according to the famous three-sigma rule.
	
	Our work focuses on robust and sample-efficient approaches to the outlier arm identification problem. We model our setting through Huber's $\epsilon$-contamination model \cite{huber1964robust} and apply robust estimators with the highest possible breakdown point \cite{donoho1983notion, rousseeuw2011robust}, i.e., median and median absolute deviation (MAD), in building the outlier threshold. Robust statistics were previously incorporated in the bandit setting \cite{altschuler2019best}, but they mainly deal with traditional settings, i.e., best arm identification, with each reward distribution being contaminated rather than identifying instance-adaptive subsets. Although our work could also be generalized to the setting with contaminated reward distribution by incorporating their techniques, we do not pursue this direction here.

	\section{Problem Setting and Notations}
	\label{sec_setting}
    We consider the standard multi-armed bandit setting where there are $n$ arms and the reward of each arm follows a 1-subgaussian distribution with mean $y_i$. The goal of the agent is to identify outlier arms whose expected rewards \emph{substantially} deviate from the \emph{majority}, in the fixed confidence and pure exploration setting. Without loss of generality, we assume $y_i \geq y_{i+1}$ and $n = 2m-1$, so that the median arm is unambiguous.\footnote{If $n=2m$, we choose the median as $m$ without loss of generality.} We also only consider identifying outliers with high rewards; identifying outliers with low rewards is analogous. 
    Let $y_{(m)} = {\median}\{y_i\}$ denote the expected reward of the \emph{median} arm, and let $\AD_i = |y_i - y_{(m)}|$ represent the absolute deviation of arm $i$ from the median. Let ${\AD_{(m)}} = {{\median}}\{|y_i - y_{(m)}|\}$ denote the Median Absolute Deviation (MAD) of expected reward. Note that $y_{(m)}$ and $\AD_{(m)}$ serve as the first two robust moments of the means of the underlying bandit instance $\{y_i\}_{i=1}^n$. We define outlier arms to be arms whose mean is greater than the threshold $\theta$ given by 
\begin{equation}
\label{eq_robust_threshold}
\theta = y_{(m)} + k \cdot {\AD_{(m)}},
\end{equation}
where $k \geq 1$ is a user-specified parameter.
The goal of the agent is to \emph{identify outlier arms using as few samples as possible}. Specifically, we are interested in designing adaptive algorithms that return the subset of outlier arms $S_o = \{i \in [n] : y_i > \theta \}$ (we assume $y_i \neq \theta, \,\forall\,i\in [n]$). We call this setting \emph{Robust Outlier Arm Identification }(ROAI). For a given error probability $\delta \in (0, 1)$, we say an algorithm is $\delta$-PAC if it correctly identifies $S_o$ with probability at least $1-\delta$ using a finite number of samples.
	
Our choice of the threshold is justified under Huber's $\epsilon$-contamination model, where with probability $1-\epsilon$ the mean $y_i$ is drawn from an unknown \emph{meta} distribution $P$ with mean $\mu$ and standard deviation $\sigma$, and with probability $\epsilon$ the mean $y_i$ is drawn from a contamination distribution. Note that sample median and MAD enjoy the highest possible breakdown point $0.5$ \citep{donoho1983notion, rousseeuw2011robust}. Hence, our threshold in \cref{eq_robust_threshold} (up to scaling of $\AD_{(m)}$) is a more robust estimator of the true threshold as compared to existing thresholds constructed using the sample mean and sample standard deviation (which have a breakdown point of $0$) \citep{zhuang2017identifying}. Furthermore, for many common meta distributions including the normal and uniform distribution, \citet{altschuler2019best} prove tight non-asymptotic concentration results for the median and MAD constructed from contaminated samples. 


    Given our assumption of $y_i \ge y_{i+1}$, let the outlier set be $S_o = \{1, \dots, n_1\}$ where $n_1$ is \emph{unknown}. 
    For a given set $\{z_i\}_{i=1}^n$, we use $z_{(k)}$ to denote the $k$-th largest value in $\{z_i\}$; particularly, we use $z_{(m)} := {\median}\{z_i\}$.

	\section{Algorithms}
	\label{sec_algorithm}
	We formally introduce our algorithms in the section. We first provide a subroutine for constructing confidence intervals (CIs) of various quantities including the outlier threshold in Section \ref{sec_confidence_interval}; and then introduce our elimination- and LUCB-style algorithms in Section \ref{sec_two_algorithms}. 

    For any arm $i \in [n]$ and time $t$, we use $S_{i, t}$ and $N_{i, t}$ to denote the sum of rewards and number of pulls; and use $\hat{y}_{i, t} = S_{i, t} / N_{i, t}$ to denote the empirical mean reward. For any quantity $q \in \{y_i, y_{(m)}, \AD_i, \AD_{(m)}, \theta\}$, we use $L_{q,t}, U_{q,t}, \mathcal{I}_{q,t}$ to denote the lower bound, upper bound, and the CI respectively of $q$ at time $t$. 
	
	\subsection{Construction of Confidence Intervals (CIs)}
	\label{sec_confidence_interval}
    The CI of individual arms $i$ can easily be constructed using Hoeffding's inequality as
    $[L_{y_i, t}, U_{y_i,t}] = [\hat{y}_{i, t} - \beta_{N_{i, t}}, \hat{y}_{i, t} - \beta_{N_{i, t}}]$,
    where $\beta_{s} = \sqrt{{\log(4ns^2/ \delta)}/{2s}}$. 

    The construction of CIs for the median ($\mathcal{I}_{y_{(m)},t}$), MAD ($\mathcal{I}_{\AD_{(m)},t}$), and the outlier threshold ($\mathcal{I}_{\theta,t}$), which are needed for ascertaining whether an arm is an outlier, is explained in Algorithm \ref{algorithm_confidence interval}. On line $1$, the CI $\mathcal{I}_{y_{(m)}, t}$ is constructed using the CIs of all arms. This is necessary because the identity of the median arm may be unknown. If the median arm can be unambiguously determined, this CI reduces to the CI of the median-th arm. The CI $\mathcal{I}_{\AD_{(m)},t}$ is similarly constructed from $\mathcal{I}_{\AD_{i},t}$. We set $\widehat{\AD}_{i, t} $ and $\hat{\theta}_{t}$ as the midpoint of their corresponding confidence intervals.

	\begin{algorithm}[]
		\caption{Construction of Confidence Intervals}
		\label{algorithm_confidence interval} 
		\renewcommand{\algorithmicrequire}{\textbf{Input:}}
		\renewcommand{\algorithmicensure}{\textbf{Output:}}
		\begin{algorithmic}[1]
			\REQUIRE CIs of individual arms $\{\mathcal{I}_{y_i, t}\}_{i=1}^n$
            \ENSURE CIs $\mathcal{I}_{y_{(m)},t}, \mathcal{I}_{\AD_{i},t}, \mathcal{I}_{\AD_{(m)},t}, \mathcal{I}_{\theta,t}$
            \vspace{5pt}
            \STATE $L_{y_{(m)}, t} = {\median} \{ L_{y_i, t} \}$ \\[2 pt]
             $U_{y_{(m)}, t} = {\median} \{ U_{y_i, t} \}$ \\[2 pt]
			$\mathcal{I}_{y_{(m)}, t} = [L_{y_{(m)}, t}, U_{y_{(m)}, t}]$
			\vspace{5pt}
			\FOR{$i = 1, \dots, n$}
			\STATE $L_{\AD_i, t} = \max \{ L_{y_i, t} - U_{y_{(m)}, t}, L_{y_{(m)}, t} - U_{y_i, t}  \}$\\[2 pt]
			$U_{\AD_i, t} = \max \{ U_{y_i, t} - L_{y_{(m)}, t}, U_{y_{(m)}, t} - L_{y_i, t} \}$ \\[2 pt]
			$\mathcal{I}_{\AD_i, t} \in [L_{\AD_i, t}, U_{\AD_i, t}]$\\[2 pt]
			$\widehat{\AD}_{i, t} = \left( U_{\AD_i, t} + L_{\AD_i, t}\right) / 2$
			\ENDFOR
			\vspace{5pt}
			\STATE $L_{\AD_{(m)}, t} = {\median} \{L_{\AD_{i},t}\} $  \\[2 pt]
			$U_{\AD_{(m), t}} = {\median} \{U_{\AD_{i},t}\}$ \\[2 pt]
			$\mathcal{I}_{\AD_{(m)}, t} = [L_{\AD_{(m)}, t}, U_{\AD_{(m)}, t}]$
			\vspace{5pt}
			\STATE $L_{\theta, t} = L_{y_{(m)}, t} + k \cdot L_{\AD_{(m)}, t}$\\[2 pt]
			$U_{\theta, t} = U_{y_{(m)}, t} + k \cdot U_{\AD_{(m)}, t}$\\[2 pt]
			$\mathcal{I}_{\theta, t} = [L_{\theta, t}, U_{\theta, t}]$ and $\hat{\theta}_{t} = \left( U_{\theta, t} + L_{\theta, t}\right) / 2$
		\end{algorithmic}
	\end{algorithm}

	\subsection{Algorithms}
	\label{sec_two_algorithms}
    We introduce our elimination-style \citep{even2006action} algorithm \roaielim and LUCB-style \cite{kalyanakrishnan2012pac} algorithm \roailucb in this section. Any pure exploration bandit algorithm is specified through its sampling, stopping, and recommendation rule \cite{kaufmann2016complexity}. The stopping and recommendation rules are the same for both algorithms. We stop at the first time $t$ such that $\{i \in [n]: \mathcal{I}_{y_i, t} \cap \mathcal{I}_{\theta, t} \neq \emptyset\}  = \emptyset$, and upon stopping we output the empirical subset of outlier arms $\hat{S}_{o, t} = \{ i \in [n]: \hat{y}_{i, t} > \hat{\theta}_t \}$. We present our two algorithms next.

    {\tt ROAIElim}: The pseudocode of \roaielim is given in Algorithm \ref{algorithm_ROAIElim}. At round $t$, \roaielim constructs three active sets for the median, the MAD, and the threshold. Each of these active sets contains arms whose CIs overlap with the respective CI. Since the threshold is constructed from the median and the MAD, any of these arms can contribute towards shrinking the CI of the threshold, and hence \roaielim samples all arms in the union of these active sets.

	\begin{algorithm}[]
		\caption{\roaielim}
		\label{algorithm_ROAIElim} 
		\renewcommand{\algorithmicrequire}{\textbf{Input:}}
		\renewcommand{\algorithmicensure}{\textbf{Output:}}
		\begin{algorithmic}[1]
			\REQUIRE Error tolerance $\epsilon$, probability of failure $\delta$, and outlier detection parameter $k$
			\ENSURE Subset of outlier arms $\hat{S}_{o, t}$
			\STATE Initialize $A_{E, 1} = A^{\median}_{E, 1} = A^{\MAD}_{E, 1} = A^{\theta}_{E, 1}= [n]$
			\FOR {$t = 1, 2, \dots$}
			\STATE Sample arms in $A_{E, t}$ and update $\{\mathcal{I}_{i, t}\}_{i \in A_{E, t}}$
			\STATE Update CIs using Algorithm \ref{algorithm_confidence interval}
			\STATE Set \vspace{-5pt}
            \begin{align*}
			A^{\median}_{E, t+1} &= \{ i \in [n] : \mathcal{I}_{y_i, t} \cap \mathcal{I}_{y_{(m)}, t} \neq \emptyset \} \cap A^{\median}_{E, t}\\[8pt]
			A^{\MAD}_{E, t+1} &= \{ i \in [n] : \mathcal{I}_{\AD_i, t} \cap
			\mathcal{I}_{\AD_{(m)}, t} \neq \emptyset \} \cap A^{\MAD}_{E, t}\\[8pt]
			A^{\theta}_{E, t+1} &= \{ i \in [n] : \mathcal{I}_{y_i, t} \cap \mathcal{I}_{\theta, t} \neq \emptyset \} \cap A^{\theta}_{E, t} \\[8 pt]
			A_{E, t+1} &= A^{\median}_{E, t+1} \cup A^{\MAD}_{E, t+1} \cup A^{\theta}_{E, t+1}
			\end{align*} 
			\STATE If $A^{\theta}_{E, t+1} = \emptyset$, stop and return $\hat{S}_{o, t}$
			\ENDFOR
		\end{algorithmic}
	\end{algorithm}

	\begin{algorithm}[h]
		\caption{\roailucb}
		\label{algorithm_ROAILucb} 
		\renewcommand{\algorithmicrequire}{\textbf{Input:}}
		\renewcommand{\algorithmicensure}{\textbf{Output:}}
		\begin{algorithmic}[1]
			\REQUIRE Error tolerance $\epsilon$, probability of failure $\delta$, and outlier detection parameter $k$
			\ENSURE Subset of outlier arms $\hat{S}_{o, t}$
			\STATE Initialize $A_{L, 1} = [n]$
			\FOR {$t = 1, 2, \dots$}
			\STATE Sample arms in $A_{L, t}$ and update $\{\mathcal{I}_{y_i, t}\}_{i \in A_{L, t}}$
			\STATE Update CIs using Algorithm \ref{algorithm_confidence interval}
			\STATE Set \begin{align*}
			A^{\median}_{L, t+1} &= \argmin_{i \in J_{m-1, t}  }\{L_{y_i, t}\} 
			\cup   \argmin_{i \in J_{m, t}} \{L_{y_i, t}\}\\
			&\cup \argmax_{i \notin J_{m-1, t}} \{ U_{y_i, t} \} 
		    \cup \argmax_{i \notin J_{m, t}} \{ U_{y_i, t} \} \\[8 pt]
			A^{\MAD}_{L, t+1} &= \argmin_{i \in J^{\AD}_{m-1, t}  }\{L_{\AD_i, t}\}
			\cup   \argmin_{i \in J^{\AD}_{m, t}} \{L_{\AD_i, t}\}\\
			&\cup \argmax_{i \notin J^{\AD}_{m-1, t}} \{ U_{\AD_i, t} \} 
			\cup \argmax_{i \notin J^{\AD}_{m, t}} \{ U_{\AD_i, t} \} \\[8 pt]
			A^{\theta}_{L, t+1} &=  \argmin_{i \in \hat{S}_{o, t}}\{L_{y_i, t} \} 
			\cup \argmax_{i \in \hat{S}_{n, t}} \{ U_{y_i, t} \}  \\
		    &\cap \{ i \in [n] : \mathcal{I}_{y_i, t} \cap \mathcal{I}_{\theta, t} \neq \emptyset \} \\[8 pt]
			A_{L, t+1} &= A^{\median}_{L, t+1} \cup A^{\MAD}_{L, t+1} \cup A^{\theta}_{L, t+1}
			\end{align*} 
			\STATE If $A_{L,t+1}^\theta = \emptyset$, stop and return $\hat{S}_{o, t} $
			\ENDFOR
		\end{algorithmic}
	\end{algorithm}

    {\tt ROAILUCB}: The pseudocode of \roailucb is presented in Algorithm \ref{algorithm_ROAILucb}. We use the notation $J_{\kappa_i, t}$ to denote $\kappa_i$ arms with the largest empirical means $\{\hat{y}_{i, t}\}$, and $J^{\AD}_{\kappa_i, t}$ to denote the $\kappa_i$ arms with the largest empirical absolute deviations $\{\widehat{\AD}_{i, t}\}$. Since we are mainly interested in shrinking confidence intervals around the median quantity, we set $\kappa_1 = m-1$ and $\kappa_2 = m$.

Motivated by the LUCB algorithm \citep{kalyanakrishnan2012pac}, \roailucb finds the $4$ arms at the median boundary, $4$ arms at the MAD boundary, and $2$ arms at the threshold boundary, and samples arms in the union of these sets. Unlike {\tt ROAIElim}, \roailucb samples at most $10$ arms in each round.

	\section{Analysis}
	\label{sec_analysis}
    In Section \ref{sec_correctness_and_sample_complexity}, we discuss correctness and sample complexity results of our algorithms. We compare the robustness and sample complexity of our algorithms with previous work in Section \ref{sec_comparison_previous}. The proofs can be found in the Appendix.
	
	\subsection{Correctness and Sample Complexity}
	\label{sec_correctness_and_sample_complexity}
	
	\cref{lem:confidence_interval} shows the correctness of CIs in Algorithm \ref{algorithm_confidence interval}. We use it to prove the correctness of our algorithms in \cref{thm_correctness}.
	\begin{restatable}{lemma}{confidenceInterval}
	\label{lem:confidence_interval}
	Suppose 
	\begin{equation*}
	\mathbb{P} \left( \forall t \in \mathbb{N}, \forall i \in [n],  y_i \in \mathcal{I}_{y_i, t} \right) \geq 1- \delta.
	\end{equation*}
	Then the CIs returned by \cref{algorithm_confidence interval} are valid with probability $1- \delta$, i.e., for $q \in \{y_{(m)}, \{ \AD_i \}_{i=1}^n, \AD_{(m)}, \theta\}$,  
	\begin{equation*}
	\mathbb{P} \left( \forall t \in \mathbb{N}, q \in \mathcal{I}_{q, t} \right) \geq 1- \delta.
	\end{equation*}
	\end{restatable}

	\begin{restatable}[Correctness]{theorem}{correctness}
			\label{thm_correctness}
	\roaielim and \roailucb are $\delta$-PAC.
	\end{restatable}

	In order to state our sample complexity bounds, we first introduce some new notations. Define
    \begin{align}
	\Delta_i^{\theta} = |\theta - y_i|, &\quad \Delta_*^{\theta} = \min_{i \in [n]} \{ \Delta_i^{\theta}\}, \nonumber \\
	\Delta_i^{\median} = |y_{(m)} - y_i|,&\quad \Delta_i^{\MAD} = | \AD_{(m)} - \AD_i |,\nonumber\\
    \Delta^*_i = \max \{\Delta_*^{\theta},  \min \{ &\Delta_i^{\theta}, \Delta_i^{\median}, \Delta_i^{\MAD} \} \}. \label{eq:sample_complexity_gap}
    \end{align}

\begin{restatable}[Sample Complexity]{theorem}{complexity}
		\label{thm_complexity}
With probability at least $1-\delta$, the sample complexity of \roaielim and \roailucb is upper bounded by 
\begin{equation}
\label{eq_complexity}
C k^2\sum_{i=1}^n   \frac{ \log \left( n k / \delta \Delta_i^{*} \right)}{(\Delta_i^{*})^2},
\end{equation}
where $C$ is a universal constant. 
\end{restatable}
	
    The sample complexity is inversely proportional to $\Delta^*_i$ defined in Eq. \eqref{eq:sample_complexity_gap}. In order to interpret the sample complexity, we consider two cases. If there exists arms whose means are close to the threshold $\theta$, i.e., $\Delta^\theta_\ast$ is small, then in order to classify these arms correctly, we need to estimate $\theta$ and consequently the median and the MAD accurately. Hence the complexity of sampling an arm depends on its gaps from $y_{(m)}, \AD_{(m)}, \theta$. Conversely, if all the arm means are widely separated from the threshold, i.e., $\Delta^\theta_\ast$ is large and there is a clear distinction between normal and outlier arms, then we do not need to estimate $\theta$ accurately, and the sample complexity is $O(n/(\Delta^\theta_\ast)^2)$.

    We highlight that the proof of \cref{thm_complexity} is non-trivial and cannot be reduced to existing techniques. The existing works \citep{kalyanakrishnan2012pac, katariya2018adaptive} deal with scenarios where the positions of the separating boundaries depend only on the arm means, and furthermore they are user-specified. This holds true only for the median in our case, it does not hold for the AD, MAD, and the threshold because their values do not depend on a single arm. The CIs of these estimators have varying degree of uncertainty and we quantify these in our Lemmas. The technical contributions may be of independent interest and we refer the reader to our proofs in the Appendix. 
    

	\subsection{Comparison to Previous Work}
	\label{sec_comparison_previous}
    We compare our setting and analysis to algorithms by \citet{zhuang2017identifying}, which is the only work study outlier detection in the bandit setting.
	
    To deal with the \emph{unknown} $\mu$ and $\sigma$, \cite{zhuang2017identifying} use the sample mean $\bar{\mu} = \sum_{i=1}^n y_i/ n$ and sample standard deviation $\bar{\sigma} = \sqrt{\sum_{i=1}^n (y_i - \bar{\mu})^2 / n}$ to approximate $\mu$ and $\sigma$, respectively, and define the outlier threshold to be $\bar{\theta} = \bar{\mu} + k \cdot \bar{\sigma}$. As discussed in \cref{sec_setting}, these estimators have a breakdown point $0$ and are very sensitive to outliers; a single extreme outlier arm can ruin their threshold.

    Algorithms developed in \citep{zhuang2017identifying} also require the reward distribution of the arms to be strictly bounded; our analysis is general and works for any sub-gaussian distributions.
    
    
    Finally, although a direct comparison of sample complexities is not possible due to different definitions of outlier thresholds, we empirically see that our algorithms require fewer samples to achieve the same error rate. 
    
    
%
	
	\section{Lower bound}
	\label{sec_lower_bound}
	
	In this section, we study lower bound on the expected number of samples needed to identify outlier arms by any $\delta$-PAC algorithm, where the outlier threshold is defined by Eq. \eqref{eq_robust_threshold}.

	Our lower bound is instance-dependent. Recall that our upper bound scales like $\tilde{O} ( \sum_{i \in [n]}{1}/{\left(\Delta_i^*\right)^2} )$ where $\Delta_i^\ast$ is given by Eq. \eqref{eq:sample_complexity_gap}. The problem is easy when $\Delta^{*}_i$ is large, and the upper bound could potentially be large when $\Delta^{\theta}_i$ is small. In this section we argue that this is unavoidable. We show that if $\Delta^{\theta}_i$ is small enough, there exists a lower bound that matches the upper bound up to logarithmic factors. This indicates that our sample complexity upper bounds are generally \emph{unimprovable}. 
	
	We apply the change of measure technique \citep{kaufmann2016complexity}, which give a lower bound in terms of the KL-divergence. To connect the KL-divergence to the Euclidean distance in our upper bound, we assume that the reward distribution of each arm is $\mathcal{N}(y_i, 1)$.\footnote{The lower bound could be generalized to other distributions, as discussed in \cite{kaufmann2016complexity}.} We use $D_{y_i}$ to denote the distribution $\mathcal{N}(y_i, 1)$ as it is fully characterized by its mean $y_i$.
	

	For a bandit instance $D_y = (D_{y_1}, \dots, D_{y_n})$, assume without loss of generality that $y_i \geq y_{i+1}$ and that each arm is unambiguously identifiable as an outlier or normal arm, i.e., $y_i \neq \theta,   \,\forall \, i\in [n]$. We use $\mathbb{E}_{y} (\cdot)$ to represent the expectation with respect to the bandit instance $D_y$ and randomness in the algorithm. We develop lower bounds for the following subset of bandit instances.
	\begin{definition}
		\label{def_minimax_lower_bound}
		Let ${\mathcal{M}}_{n, \rho} = \{ D_y = (D_{y_1}, \dots, D_{y_n}): y_i \neq \theta \}$ be a subset of bandit instances with $\theta$ defined in Eq. \eqref{eq_robust_threshold} and $k \geq 2$, and satisfying the following two conditions.
		\begin{enumerate}
			\item There exists a unique median $y_{(m)}$ and a unique MAD $\AD_{(m)}$, with 
			\begin{equation*}
			\eta := {1}/{2} \cdot  \min_{i \in \{ m, m-1\}} \left\{ y_{(i)} - y_{(i+1)}, \AD_{(i)}-\AD_{(i+1)}  \right\} .
			\end{equation*}
			
        \item There exists a constant $\rho <  \eta$ such that at least two arms $l_1$ and $l_2$ such that ${\rho}/{2}< \theta - y_{l_i}  <\rho$, and at least two arms $u_1$ and $u_2$ such that ${\rho}/{2}<   y_{u_i}- \theta  < \rho $; furthermore, there exists no arm with mean in $[\theta - \rho/2, \theta + \rho/2]$.

			
		\end{enumerate}
	\end{definition}
	
	It is easy to see that $\mathcal{M}_{n, \rho} \neq \emptyset$ for reasonably large $n$. The conditions in \cref{def_minimax_lower_bound} are essentially to make sure that slightly changing the median $y_{(m)}$ or the MAD $\AD_{(m)}$ will incur a change in the set of outlier arms. Then, for any $\delta$-PAC algorithm to correctly identify the subset of outlier arms, it is necessary to accurately identify the outlier threshold, which eventually leads to a matching sample complexity lower bound. We state our lower bound for the subset of bandit instances ${\mathcal{M}}_{n, \rho}$ next.
	

	\begin{restatable}{theorem}{lowerBound}
		\label{thm_lower_bound_worst_case}
Suppose bandit instance $D_y \in {\mathcal{M}}_{n, \rho}$. Then for $\delta \leq 0.15$, any $\delta$-PAC outlier arm identification algorithm $\mathcal{A}$ with outlier threshold constructed as in Eq. \eqref{eq_robust_threshold} and an almost surely finite stopping time $\tau$, we have that
\begin{equation*}
\mathbb{E}_y[\tau] \geq  \sum_{i \in [n]} \frac{1}{5  \left( {{\Delta}}^{*}_{i}  \right)^2}  \log \left( \frac{1}{2.4 \delta}\right).
\end{equation*}
	\end{restatable}

%
%
	
    In general for bandit instances outside $\mathcal{M}_{n,\rho}$ but with non-empty subset of outlier arms, the outlier identification problem is at least as hard as the top-$n_1$ arm identification problem where $n_1$ is the number of outlier arms \emph{given} by an oracle. Thus, any lower bound for top-$n_1$ arm identification, e.g., Theorem 4 in \cite{kaufmann2016complexity}, applies as a general lower bound for the outlier arm identification problem.

	\section{Heuristic to Reduce Sample Complexity}
	\label{sec_generalization}
	
	The sample complexity of our algorithms is inversely proportional to $(\Delta_i^{*})^2$ (see Eq. \eqref{eq:sample_complexity_gap}), which could be as small as $(\min \{ \Delta_i^{\theta}, \Delta_i^{\median}, \Delta_i^{\MAD} \})^2 $ if $\Delta_*^{\theta}$ is small. As $n$ increases, there can be many arms with small $\Delta_i^{\median}$ or $ \Delta_i^{\MAD}$ and the sample complexity can be high as a result. In general, we cannot circumvent this cost if the outlier threshold is constructed as in \cref{eq_robust_threshold}.
	    
	However, it might not be necessary to always construct outlier threshold using all $n$ arms, and one heuristic approach is to construct threshold only from a subset of arms. Suppose we know, from an oracle, an upper bound $c < 0.5$ on the fraction of arms drawn from the contaminated distribution,  we could then randomly draw a subset $\Omega \subseteq [n]$ of arms with cardinality $|\Omega| \geq 2 \lfloor n c \rfloor  +1$. The cardinality requirement makes sure the fraction of contamination within the subset $\Omega$ is smaller than $0.5$ so that the median and MAD are not arbitrarily destroyed by outliers; but of course the threshold constructed crucially depends on the selection of $\Omega$. Although the outlier set computed from this modified threshold could differ from the outlier set computed from $[n]$, we could potentially enjoy a smaller sample complexity. We next state an upper bound on the sample complexity in this setting.\footnote{See Appendix \ref{appendix_heuristic} for details of the algorithm.} Empirical examinations of the performance are summarized in \cref{sec_setting_comparison}.

	\begin{restatable}{corollary}{subSampling}
		\label{thm_subsampling}

Suppose we run \cref{algorithm_ROAILucb} with $y_{(m)}$, $\AD_{(m)}$ and $\theta$ constructed using arms in $\Omega \subseteq [n]$. Then, with probability at least $1 - \delta$, the sample complexity is upper bounded by 
\begin{equation*}
Ck^2 \sum_{i\in \Omega}   \frac{\log \left(  {nk}/({\delta \Delta_i^{*} })  \right) }{(\Delta_i^{*})^2}  + C \sum_{i\notin \Omega}   \frac{ \log \left(  {n }/({\delta \Delta_i^{\theta} })  \right) }{(\Delta_i^{\theta})^2},
\end{equation*}
where $\Delta^*_i = \max \{\Delta_*^{\theta},  \min \{ \Delta_i^{\theta}, \Delta_i^{\median}, \Delta_i^{\MAD} \} \}$ and $C$ is a universal constant.
	\end{restatable}

	\section{Experiments}
	\label{sec_experiment}
    We conduct three experiments. In \cref{sec:exp_sample_complexity}, we verify the tightness of our sample complexity upper bounds in \cref{sec_correctness_and_sample_complexity}. In \cref{sec_setting_comparison}, we compare our setting to the non-robust version proposed by \citet{zhuang2017identifying} and empirically confirm the robustness of our thresholds as discussed in \cref{sec_comparison_previous}. Finally, in \cref{sec_anytime_performance}, we compare the anytime performance of our algorithms with baselines on a synthetic and a real-world dataset. For ease of comparison, we make the fraction of contamination deterministic rather than random as in the original Huber's contamination model. All our results are averaged over 500 runs. Error bar in \cref{fig_deviation}, \cref{fig_synthetic_data} and \cref{fig_real_data} are rescaled by $2/\sqrt{500}$. Our code is publicly available \citep{ROAIcode}.
    

	\subsection{Sample Complexity}
    \label{sec:exp_sample_complexity}
    
    	\begin{figure}[h]%
    	\centering
    	\subfigure[]{{\includegraphics[width=0.4\textwidth]{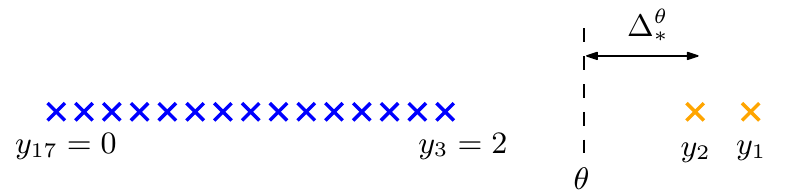} }}\\
    	\subfigure[]{{\includegraphics[width=0.48\textwidth]{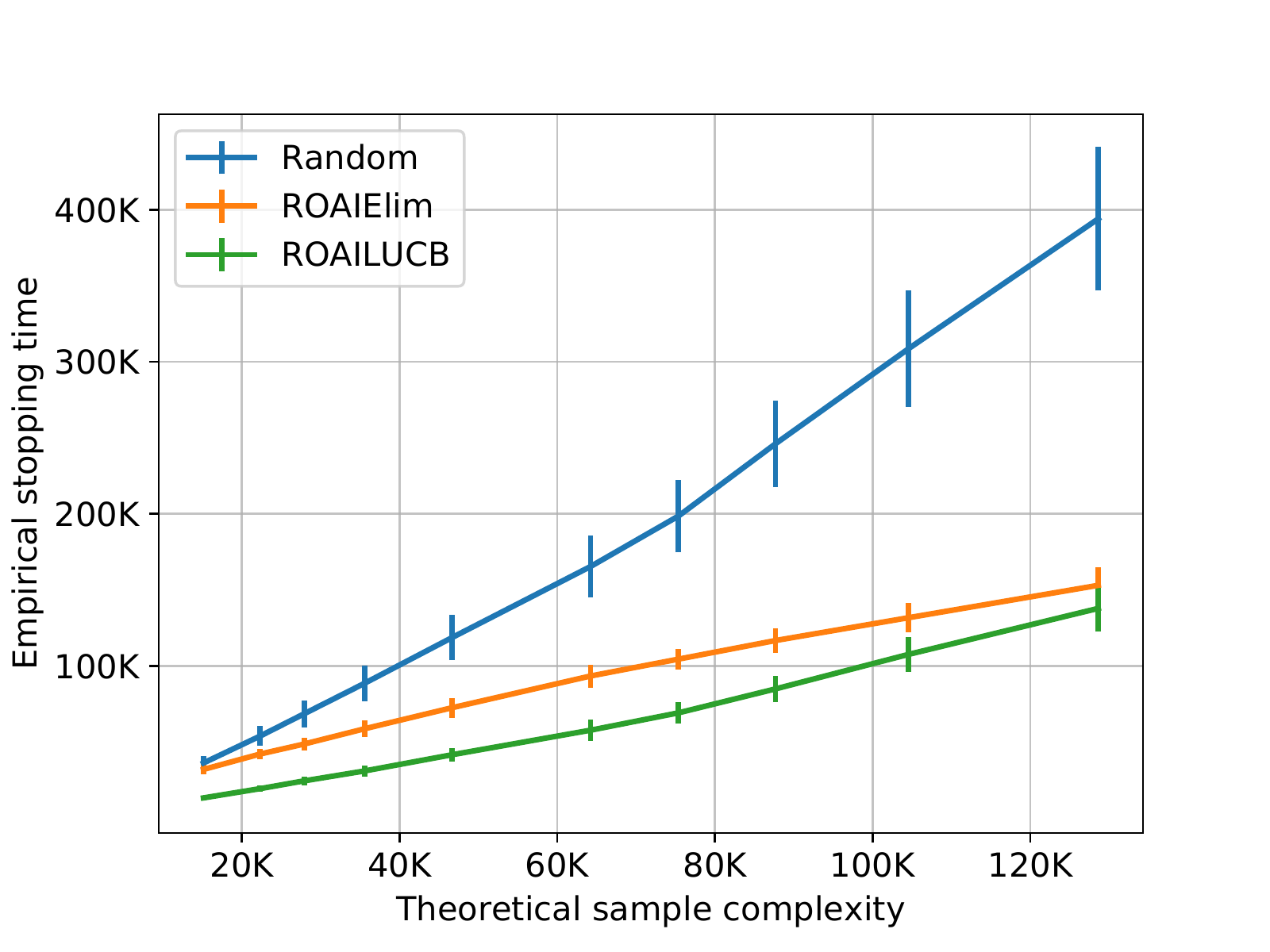} }}%
    	\caption{(a) Configuration of the arm means, we vary $\Delta^\theta_\ast$ to change hardness (b) Theoretical upper bound vs empirical stopping time, the linear relationship shows that our upper bounds are correct up to constants.}%
    	\label{fig_complexity}
    \end{figure}

	In order to test that the hardness predicted by our upper bound scales correctly, we first plot the empirical stopping time of each algorithm against the theoretical sample complexity (\cref{thm_complexity} with $C=10$). We choose the arm configuration in \cref{fig_complexity}(a) containing 15 normal arms (in blue) with fixed means equally distributed from 0 to 2, an outlier threshold $\theta \approx 2.837$, and 2 outlier arms (in orange) above the outlier threshold. The distance between the outlier arms is fixed at $0.2$. We decrease $\Delta_*^\theta$ from $0.6$ to $0.2$, and this changes the theoretical sample complexity. Note that the threshold does not change. The reward of each arm is normally distributed with standard deviation $0.5$. In \cref{fig_complexity}(b), we plot the empirical stopping time of our algorithms against the theoretical sample complexity, and we see a linear relationship between the two, which suggests that our sample complexity in \cref{thm_complexity} is correct up to constants. \cref{fig_complexity}(b) also shows that our adaptive algorithms always outperform random sampling, and the gains increase with the hardness of the problem.

	\subsection{Setting Comparison}
	\label{sec_setting_comparison}
	
    In this section, we compare the robustness of our outlier threshold and the sample complexity upper bound of our algorithms to the threshold and algorithms considered by \citet{zhuang2017identifying}. We introduce the nomenclature of the algorithms next. Round Robin (RR) and Weighted Round Robin (WRR) are algorithms proposed by \citet{zhuang2017identifying} which use a non-robust outlier threshold. We denote by \roai-$\lambda n\epsilon$ the algorithm suggested in Section \ref{sec_generalization} that constructs the outlier threshold from a subset $\Omega$ of arms with $|\Omega| = \max\{\lambda \lfloor n \epsilon \rfloor +1, 15\}$. For each run of this experiment, we generate the means of normal arms from $\mathcal{N}(0.3, 0.075^2)$ (clipped to the three-sigma range), and the means of outlier arms from \verb+Unif+$(x, 1)$. We draw 105 arms in total. We multiply MAD with $1/(\Phi^{-1}(3/4)) \approx 1.4826$ to make it consistent for the true scale of normal distribution \citep{leys2013detecting}.

	\begin{figure}[h]
		\centering
		{{\includegraphics[width=\linewidth]{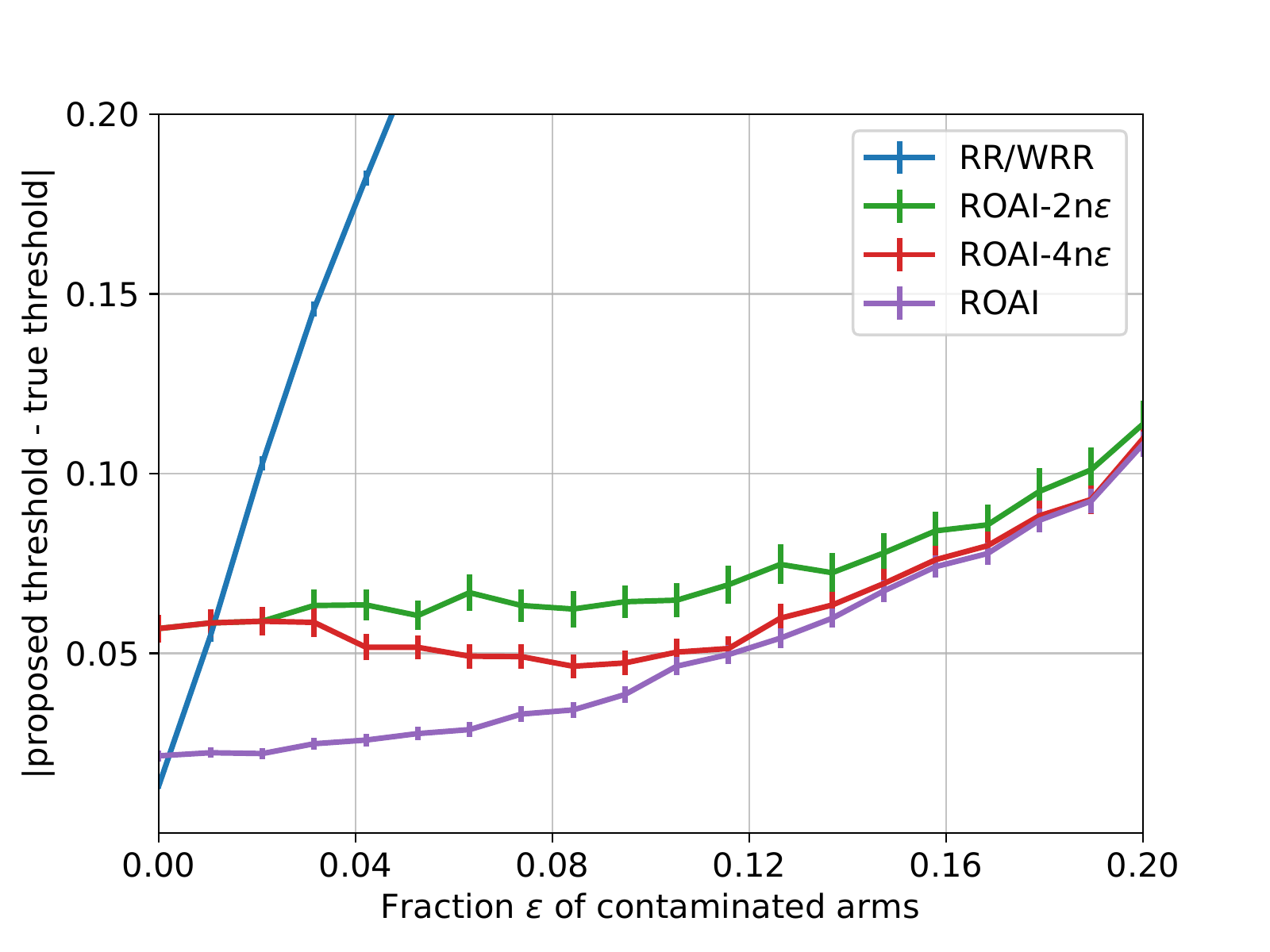} }}%
		\caption{Deviation of the proposed outlier threshold from the true threshold as a function of the contamination level $\epsilon$. This shows that our threshold is robust to contamination.}%
		\label{fig_deviation}
	\end{figure}

	We first study robustness. In \cref{fig_deviation}, we generate outlier arms from \verb+Unif+(0.7, 1) and vary the fraction $\epsilon$ of contaminated arms from 0 to 0.2, and compare the robustness of the proposed outlier threshold from different algorithms. We measure the robustness as deviation of the proposed threshold from the true threshold. The true threshold is chosen according to the three-sigma rule. It is clear that our outlier thresholds are much more robust to contamination. 

    We next compare the upper bounds on the sample complexity of different algorithms. We generate 10 outlier arms from \verb+Unif+$(x, 1)$ with $x$ varying from $0.6$ to $0.9$. In \cref{fig_complexity_upper_bound}, we plot the median sample complexity upper bounds of each algorithm in log scale, ignoring universal constants. We notice that under these contamination settings, our sample complexity upper bounds are orders of magnitude smaller than the ones proposed in \citet{zhuang2017identifying}. From \cref{fig_deviation} and \cref{fig_complexity_upper_bound}, we also see the trade-off between robustness and sample complexity for our generalized algorithms suggested in \cref{sec_generalization}.
    

   
	
	\begin{figure}[h]
		\centering
		{{\includegraphics[width=\linewidth]{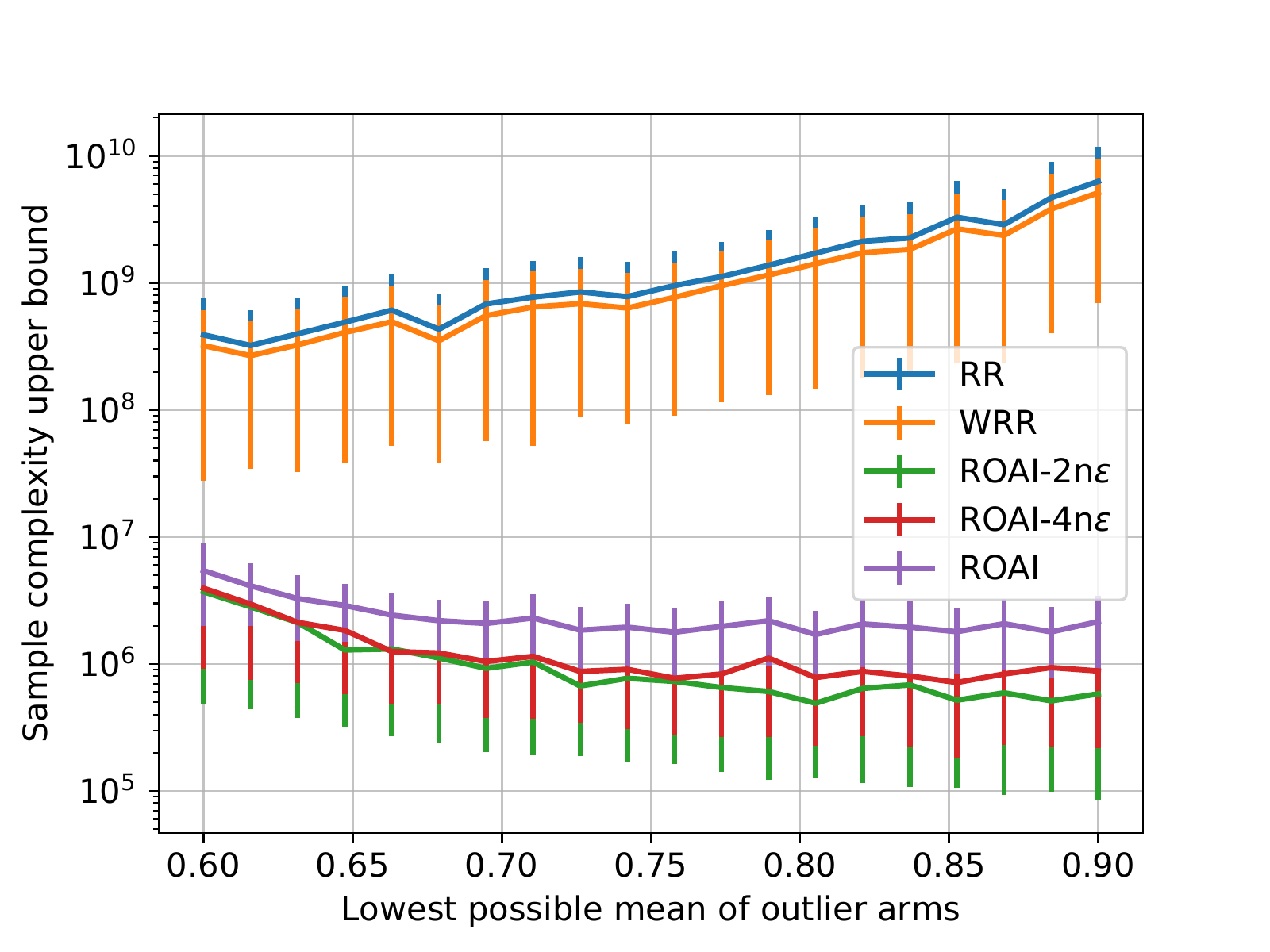} }}%
		\vspace{- 6pt}
        \caption{Sample complexity upper bounds as a function of the lowest possible mean of outlier arms, our upper bounds are smaller.}
		\label{fig_complexity_upper_bound}
	\end{figure}

	\subsection{Anytime Performance}
	\label{sec_anytime_performance}

    In this section, we examine the anytime empirical error rate of {\tt ROAILUCB}, {\tt ROAIElim}, random sampling and \textsc{RR/WRR} \citep{zhuang2017identifying}. Similar to \cref{sec_setting_comparison}, we generate $100$ normal arm means from $\mathcal{N}(0.3, 0.075^2)$ and $5$ outlier means from \verb+Unif+(0.8, 1). We draw rewards of each arm from a Bernoulli distribution with respect to its mean. We use Bernoulli distributions here as algorithms in \citet{zhuang2017identifying} only apply to arms with a strictly bounded distribution. In order to simulate a run, we randomly draw means according to these two distributions and then draw rewards from these arms with fixed means till the end of the run. Under this setting, both our threshold (median-MAD) and the threshold in \citet{zhuang2017identifying} (mean-standard deviation) will lie in $[0.525, 0.8]$ with high probability.  We filter out instances where the outlier sets (with respect to both thresholds) do not match the ground truth. {The averaged minimum gap $\min \{|y_i - \theta|\}$ is $0.062$ according to our threshold, and $0.063$ according to theirs.}
    In \cref{fig_synthetic_data}, we plot the fraction of times any algorithm fails to identify the correct set of outlier arms. We notice that \roailucb requires about 5x fewer samples than RR/WRR for the same error rate. Notice that RR is essentially random sampling with their threshold, and hence we use our threshold in the algorithm labeled Random. The empirical performance of RR/WRR is worse than Random.
    
    
    
%

	\begin{figure}[h]
		\centering
		{{\includegraphics[width=\linewidth]{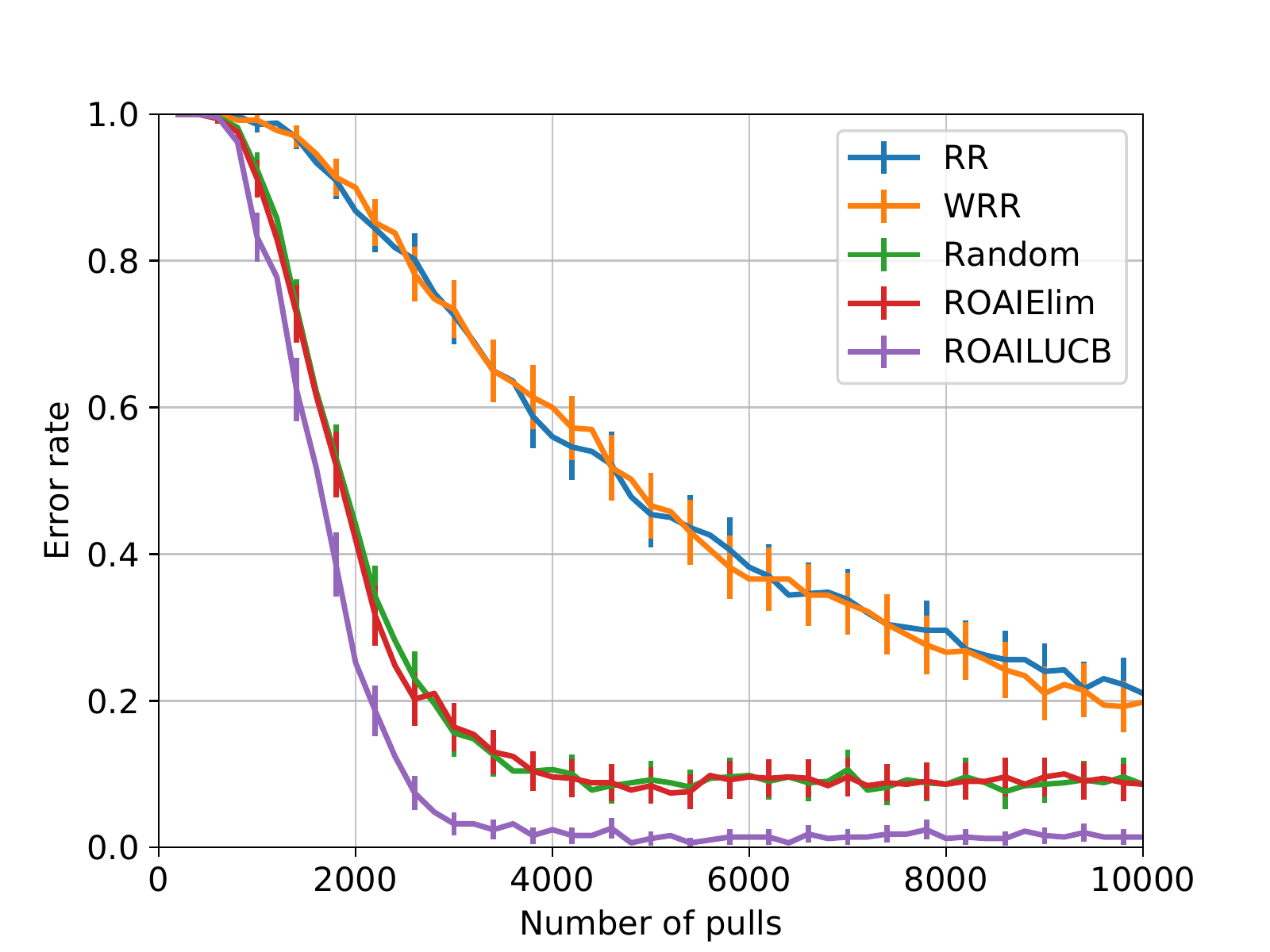} }}%

        \caption{Fraction of times the outlier set is misidentified on synthetic data.}
		\label{fig_synthetic_data}
	\end{figure}
	
	We also compare the performance of all algorithms on the real-world Wine Quality dataset \citep{sathe2016lodes}, which is widely used to compare outlier detection algorithms. This dataset contains $129$ wines, each having $13$ features. $10$ of these wines are labeled as outliers in the dataset. To obtain a 1d representation of each wine, we projected data points on the first principal component and then rescaled them to $[0, 1]$. We deleted 6 values closest to the threshold in this 1d representation so that the outlier set is the same according to both definitions. The 123 means thus obtained are plotted in \cref{fig_real_data}(a) with the top-5 outliers in orange. We simulate each arm as a Bernoulli distribution. As in the previous experiment, \roailucb greatly outperform other algorithms, and RR/WRR is worse than random sampling.
	
	\begin{figure}[h]%
		\centering
		\subfigure[]{{\includegraphics[width=.4\textwidth]{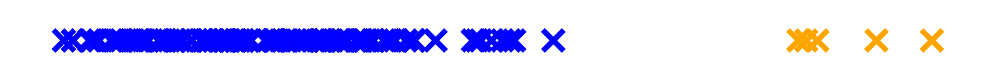} }}\\

		\subfigure[]{{\includegraphics[width=.48\textwidth]{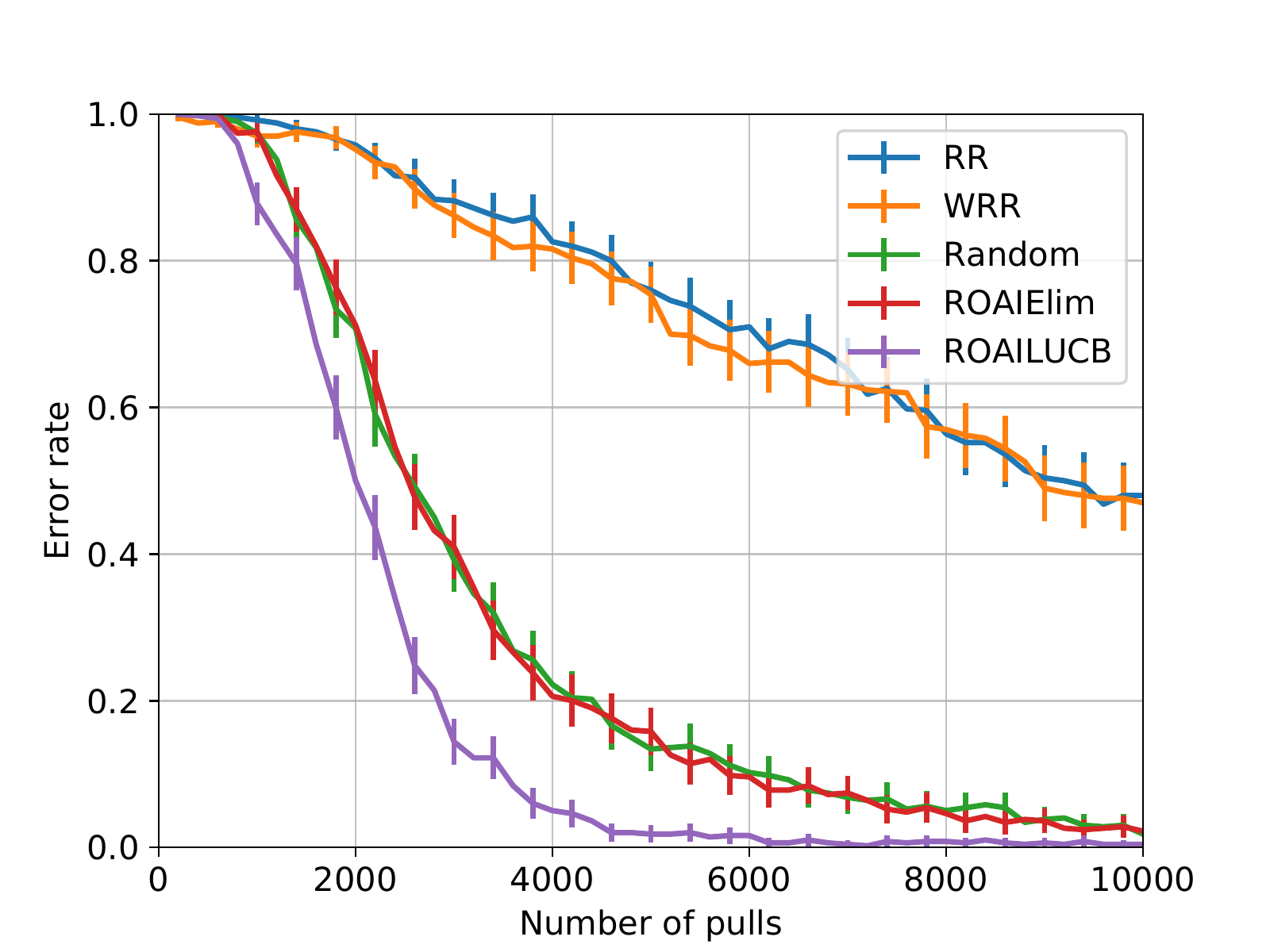} }}%
        \caption{(a) 1d means obtained from the Wine Quality dataset (b) Fraction of times the outlier set is misidentified on this dataset.}
		\label{fig_real_data}
	\end{figure}

	The fact that \roaielim performs similar to random sampling in terms of the anytime error rate is not new \citep{jamieson2014lil}, elimination-style algorithms are very conservative initially. \cref{fig_complexity}(b) does show that \roaielim outperforms random sampling in terms of the empirical stopping time.
	
	
	\section{Conclusion}
	\label{sec_conclusion}
	
    This paper studies robust outlier arm identification problem, a pure exploration problem with instance-adaptive identification target in the multi-armed bandit setting. We propose two algorithms \roaielim and {\tt ROAILUCB}, and theoretically derive their correctness and sample complexity upper bounds. We also provide a matching, up to log factors, worst case lower bound, indicating our upper bounds are generally tight. We conduct experiments to show our algorithms are both robust and about 5x sample efficient compared to state-of-the-art. 
    

    We leave open several questions. First, the sample complexity of our algorithms is large when $\Delta^\theta_*$ is small. We propose a heuristic to partially address this issue if an upper bound on the contamination $\epsilon$ is known in \cref{sec_generalization}. Another potential approach is to add an error tolerance to allow arms close the threshold being misclassified, but that adds another user-specific parameter. We also leave open the problem of obtaining a tight instance dependent lower bound. Our current lower bound, even though instance-dependent, works only in the worst case, and we reduce the problem to top-$n_1$ arm identification in the general case. 
    
    \section*{Acknowledgements}
        Yinglun Zhu would like to thank Ardhendu Tripathy for helpful discussions, and thank Tuan Dinh for help with experiments in the early stage of this project.

	\bibliography{ROAI}
	\bibliographystyle{icml2020}
	
	\clearpage
	
	\onecolumn
	\appendix
	
	\section{Correctness Analysis: Proof of \cref{lem:confidence_interval} and \cref{thm_correctness}}
	
	We first define two events $W^\prime$ and $W$ as followings:
	\begin{equation*}
	W^{\prime} = \bigcap_{t \in \mathbb{N}}\bigcap_{i \in [n]} \{ y_i \in \mathcal{I}_{y_i, t} \}
	\end{equation*}
	and
	\begin{equation*}
	W = \bigcap_{t \in \mathbb{N}} \left\{ \bigcap_{i \in [n]} \{ y_i \in \mathcal{I}_{y_i, t} \}  \bigcap_{q \in \{y_{(m)}, \{\AD_i\}_{i=1}^n, \AD_{(m)}, \theta \}} \{ q \in \mathcal{I}_{k, t} \} \right\}.
	\end{equation*}
	
	A byproduct of the proof of Lemma \ref{lem:confidence_interval} is $W^\prime \implies W$.

\confidenceInterval*

	\begin{proof}
		Under event $W^\prime$, which is assumed to hold with probability at least $1-\delta$, we show that the construction of confidence intervals through Algorithm \ref{algorithm_confidence interval} are valid and tight.\footnote{We set $L_{\AD_i, t} = \max \{ L_{y_i, t} - U_{y_{(m)}, t}, L_{y_{(m)}, t} - U_{y_i, t}  \}$ in \cref{algorithm_confidence interval} in order to provide better estimations of $\widehat{\AD}_{i, t}$ in experiments, especially when $t$ is small. Here, we actually prove with a slightly tighter version $L_{\AD_i, t} = \max \{0,  L_{y_i, t} - U_{y_{(m)}, t}, L_{y_{(m)}, t} - U_{y_i, t}  \}$, as absolute deviation is always non-negative. Our analysis works for the situations $L_{\AD_i, t} = \max \{ L_{y_i, t} - U_{y_{(m)}, t}, L_{y_{(m)}, t} - U_{y_i, t}  \}$ as well: (1) for {\tt ROAIElim}, one could simply change \cref{def_FindAD}, and the subsequent \cref{lm_FindAD} and \cref{lm_extended_ci_MAD} still hold; (2) for {\tt ROAILUCB}, \cref{lm_length_confidence_bound_AD} still hold.}

		We first show $y_{(m)} \in \mathcal{I}_{y_{(m)}, t}$ through contradiction: Suppose $y_{(m)} < L_{y_{(m)}, t}$, we then know there exists at least $m$ arms with means greater than $y_{(m)}$, which contradicts the fact that there could have at most $m-1$ arms with means greater than $y_{(m)}$; similarly, we know that $y_{(m)} > U_{y_{(m)}, t}$ cannot hold. We thus know $y_{(m)} \in \mathcal{I}_{y_{(m)}, t}$. The confidence interval of $y_{(m)}$ is tight in the sense that we could have $y_{(m)} = L_{y_{(m)}, t}$ or $y_{(m)} = U_{y_{(m)}, t}$: It is possible that all means of arms are on the right or left end points of their corresponding confidence intervals.
		
		We next show $\AD_i \in [L_{\AD_i, t}, U_{\AD_i, t}]$. We first argue $\AD_i \geq L_{\AD_i, t}$: When $\mathcal{I}_{y_i, t} \cap \mathcal{I}_{y_{(m)}, t} \neq \emptyset$, we have both $L_{y_i, t} - U_{y_{(m)}, t} \leq 0$ and $ L_{y_{(m)}, t} - U_{y_i, t} \leq 0$; since $\AD_i \geq 0$ by definition, we have $\AD_i \geq L_{\AD_i, t} = 0$. When $\mathcal{I}_{y_i, t} \cap \mathcal{I}_{y_{(m)}, t} = \emptyset$, we have one of $L_{y_i, t} - U_{y_{(m)}, t} $ and $ L_{y_{(m)}, t} - U_{y_i, t} $ represents the smallest distance between two points within their corresponding confidence intervals and the other one represents the negative of the largest distance between two points within their corresponding confidence intervals; since $\AD_i$ cannot be smaller than the smallest distance between two confidence intervals, we have $\AD_i \geq L_{\AD_i, t}$. To prove $\AD_i \leq U_{\AD_i, t}$: We notice that no matter two confidence intervals overlap or not, we have $\max \{ U_{y_i, t} - L_{y_{(m)}, t}, U_{y_{(m)}, t} - L_{y_i, t} \}$ represents the largest distance between two points within their corresponding confidence intervals; thus, $\AD_i \leq U_{\AD_i, t}$. The confidence interval $\mathcal{I}_{\AD_i, t}$ is tight in the sense that we could have $\AD_i = U_{AD_i, t}$ or $\AD_i = L_{AD_i, t}$.
		
		The proof of $\AD_{(m)} \in \mathcal{I}_{\AD_{(m), t}}$ is similar to the proof of $y_{(m)} \in \mathcal{I}_{y_{(m)}, t}$; and the proof of $\theta \in \mathcal{I}_{\theta, t}$ follows from the fact that $\theta = y_{(m)} + k \cdot \AD_{(m)}$ and the construction of $\mathcal{I}_{y_{(m)}, t}$ and $\mathcal{I}_{\AD_{(m), t}}$.
		
		We notice that above analyses show $W^\prime \implies W$. Thus, $\mathbb{P} (\neg W) \leq \mathbb{P} (\neg W^\prime) \leq \delta$.
	\end{proof}

\correctness*

	\begin{proof}
		We first notice that, under event $W$, when $A^{\theta}_{E, t} = \emptyset$ or $A^{\theta}_{L, t} = \emptyset$, we have the confidence interval of all arms being separated from the confidence interval of outlier threshold; and thus we could correctly output the subset of outlier arms.
		
		We next show that $\mathbb{P}(W) \geq 1- \delta$: Since we have $W^\prime \implies W$, which leads to $\mathbb{P} (\neg W) \leq \mathbb{P} (\neg W^\prime)$; we thus only need to show $\mathbb{P} (\neg W^\prime) < \delta$ in the following. Based on Hoeffding's inequality and the construction of confidence bound for individual arms, we actually allocate $\delta/(2n N_{i, t}^2)$ failure probability to arm $i$ at the $N_{i, t}$-th pull. The fact that 
		\begin{equation*}
		\sum_{i \in [n]} \sum_{N_{i, t} = 1}^{\infty} \frac{\delta}{2n N_{i, t}^2} = \frac{\pi^2 \delta}{12} < \delta
		\end{equation*}
		directly leads to the desired result $\mathbb{P} (\neg W^\prime) < \delta$.
		
		We prove that both algorithms stop in finite time in our sample complexity analysis.
	\end{proof}

	\section{Sample Complexity Analysis: Restate of \cref{thm_complexity}}
	\label{appendix_sample_complexity}
	
	Recall
	 \begin{align*}
	\Delta_i^{\theta} = |\theta - y_i|, &\quad \Delta_*^{\theta} = \min_{i \in [n]} \{ \Delta_i^{\theta}\}, \nonumber \\
	\Delta_i^{\median} = |y_{(m)} - y_i|,&\quad \Delta_i^{\MAD} = | \AD_{(m)} - \AD_i |,\nonumber\\
	\Delta^*_i = \max \{\Delta_*^{\theta},  \min \{ &\Delta_i^{\theta}, \Delta_i^{\median}, \Delta_i^{\MAD} \} \}.
	\end{align*}
		
	We restate Theorem \ref{thm_complexity} as below. The proofs of \roaielim and \roailucb can be found in Appendix \ref{appendix_ROAIELim} and Appendix \ref{appendix_ROAILucb}, respectively. The factor $k$ appears in the upper bound in \cref{thm_complexity} as $ \mathcal{I}_{\AD_{(m)}, t}$ is enlarged by a factor of $k$ when constructing $\mathcal{I}_{\theta, t}$.

	
	\complexity*

	We also provide the following remark, which will be referred frequently in our formal sample complexity analysis.
	\begin{remark}
		\label{remark_sample_needed}
		There exists a universal constant $C$ such that for any $\lambda > 0$, when $s > C \frac{\log (n/\delta \lambda)}{\lambda^2}$, we have $\beta_s < \lambda$ with
		\begin{equation*}
		\beta_{s} = \sqrt{\frac{\log(4ns^2/ \delta)}{2s}}.
		\end{equation*}
	\end{remark}

	\section{Sample Complexity Analysis of \roaielim}
	\label{appendix_ROAIELim}
	
	We conduct the sample complexity analysis of \roaielim on top of intersecting confidence intervals introduced in Appendix \ref{appendix_intersecting_confidence_interval}.
	
	We start by providing some supporting lemmas; we then characterize the
	confidence interval of the median $y_{(m)}$, the MAD $\AD_{(m)}$ and the outlier threshold $\theta$ before getting into the sample complexity analysis. We will upper bound sample complexity under the good event $W$, which happens with probability of at least $1-\delta$.
	
	\subsection{Intersecting Confidence Interval}
	\label{appendix_intersecting_confidence_interval}
	
	Start from now on and up to the proof of sample complexity of \roaielim (Theorem \ref{thm_complexity_Elim}), we will conduct our analysis with respect to the following intersecting confidence intervals:
	\begin{equation}
	\label{eq_intersecting_confidence_interval}
	\mathcal{I}^{\prime}_{y_i, t} = \bigcap_{t^\prime \leq t} \mathcal{I}_{y_i, t^{\prime}}   = \left[ \max_{t^\prime \leq t}L_{y_i, t^\prime}, \min_{t^\prime \leq t}U_{y_i, t^{\prime}} \right] =: \left[ L^\prime_{y_i, t},  U^\prime_{y_i, t}\right]
	\end{equation}
	for $\forall t \in \mathbb{N}, i \in [n]$. It's easy to see that $\mathcal{I}^{\prime}_{y_i, t} $ is a valid confidence interval and $\mathcal{I}^{\prime}_{y_i, t} \subseteq \mathcal{I}_{y_i, t} = \left[ L_{y_i, t},  U_{y_i, t}\right]$, a property will be used frequently. A formal analysis for the correctness of $\mathcal{I}^{\prime}_{y_i, t}$ could be obtained from Lemma 9 in \cite{katariya2019maxgap}. We apply $\mathcal{I}^{\prime}_{y_i, t}$ in \cref{algorithm_confidence interval} and \cref{algorithm_ROAIElim}. 
	

	We analyze confidence intervals of the median, the median absolute deviation and the threshold, with the help of this tighter confidence interval. For convenience, we will keep using notation $\mathcal{I}_{y_i, t} = \left[ L_{y_i, t},  U_{y_i, t}\right]$ to actually represent $\mathcal{I}^{\prime}_{y_i, t}= \left[ L^\prime_{y_i, t},  U^\prime_{y_i, t}\right]$; and $\mathcal{I}_{y_{(m)}, t}$, $\mathcal{I}_{\AD_{i}, t}$, $\mathcal{I}_{\AD_{(m)}, t}$, $\mathcal{I}_{\theta, t}$ to represent the generated confidence intervals through Algorithm \ref{algorithm_confidence interval} with the intersecting confidence intervals as input. There will be no ambiguity of this slightly abuse of notation as original confidence intervals are not used anymore up to the end of Appendix \ref{appendix_ROAIELim}.
	
	\subsection{Supporting Definition and Lemmas}
	
	\begin{lemma}
		\label{lm_extended_ci}
		$\forall i \in [n]$ and $\forall t \in \mathbb{N}$, we have
		\begin{equation*}
		[L_{y_i, t}, U_{y_i, t}] \subseteq [y_i - 2 \beta_{N_{i, t}}, y_i + 2 \beta_{N_{i, t}}].
		\end{equation*}
	\end{lemma}
	\begin{proof}
		We only need to prove $L_{y_i, t} \geq y_i - 2 \beta_{N_{i, t}}$ and the other side is symmetric. Combining $L_{y_i, t} \geq \hat{y}_{i, t} - \beta_{N_{i, t}}$ with the fact that $\hat{y}_{i, t} \geq y_i - \beta_{N_{i, t}}$ leads to the desired result.\footnote{Recall we use $L_{y_i, t}$ to present the lower bound of intersecting confidence intervals; and that's why we have `$L_{y_i, t} \geq \hat{y}_{i, t} - \beta_{N_{i, t}}$' rather than `$L_{y_i, t} = \hat{y}_{i, t} - \beta_{N_{i, t}}$'.}
	\end{proof}
	
	\begin{definition}
		\label{def_FindAD}
		We define a $\findad\bigl([l_1, u_1], [l_2, u_2]\bigr)$ operator with two confidence intervals as input and output a single confidence interval of absolute deviation as following
		\begin{equation}
		\label{eq_FindAD}
		\findad\bigl([l_1, u_1], [l_2, u_2]\bigr) = [\max \{0, l_1 - u_2, l_2 - u_1\}, \max\{ u_1 - l_2, u_2 - l_1 \}].
		\end{equation} 
		Note that $\findad$ is symmetric with respect to its inputs.
	\end{definition}
	
	\begin{lemma}
		\label{lm_FindAD}
		
		Suppose $[l_1, u_1] \subseteq [l_1^\prime, u_1^\prime]$, then we have $\findad\bigl([l_1, u_1], [l_2, u_2]\bigr) \subseteq \findad\bigl([l_1^\prime, u_1^\prime], [l_2, u_2]\bigr)$.
	\end{lemma}
	
	\begin{proof}
		This is almost self-evident by combining Eq. \eqref{eq_FindAD} with the fact that $u_1^\prime \geq u_1$ and $l_1^\prime \leq l_1$.
	\end{proof}
	
	\begin{lemma}
		\label{lm_loc_comparison}
		Given two set $\{a_i\}_{i=1}^n$ and $\{b_i\}_{i=1}^n$. If $a_i \geq b_i$ for each $i \in [n]$, then we have $a_{(j)} \geq b_{(j)}$ for any $j \in [n]$.
	\end{lemma}
	
	\begin{proof}
		We prove the result through contradiction. Suppose $a_{(j)} < b_{(j)}$, then there exists a subset $S \subseteq [n]$ with $|S| \geq n-j+1$ such that $\forall i \in S, a_{i} < b_{(j)}$; this results in at most $j-1$ items among $\{a_i\}_{i=1}^{n}$ are greater than or equal to $b_{(j)}$. On the other side, since $a_i \geq b_i$ and $b_{(j)}$ is the $j$-th largest item among $\{b_i\}_{i=1}^n$, there should have at least $j$ items among $\{a_i\}_{i=1}^n$ being greater or equal to $b_{(j)}$, which leads to a contradiction.
	\end{proof}

\begin{lemma}
\label{lm_other_intersecting_confidence_interval}
Let $\mathcal{I}_{y_i, t} = [L_{y_i, t}, U_{y_i, t}]$ represents the intersecting confidence intervals of arms $i$. Suppose $\mathcal{I}_{y_{(m)}, t}$,  $\mathcal{I}_{\AD_i, t}$, $\mathcal{I}_{\AD_{(m)}, t}$ and $\mathcal{I}_{\theta, t}$ are generated from Algorithm \ref{algorithm_confidence interval} with input $\{\mathcal{I}_{y_i, t}\}_{i=1}^{n}$, then for any $t^\prime \leq t$, we have 
\begin{align*}
\mathcal{I}_{y_{(m)}, t}  \subseteq \mathcal{I}_{y_{(m)}, t^\prime}, & \quad \mathcal{I}_{\AD_i, t}  \subseteq \mathcal{I}_{\AD_i, t^\prime},\\
\mathcal{I}_{\AD_{(m)}, t}  \subseteq \mathcal{I}_{\AD_{(m)}, t^\prime}, & \quad \mathcal{I}_{\theta, t}  \subseteq \mathcal{I}_{\theta, t^\prime}.\\
\end{align*}
\end{lemma}
\begin{proof}
We first notice $[L_{y_i, t}, U_{y_i, t}] \subseteq [L_{y_i, t^\prime}, U_{y_i, t^\prime}]$ according to the construction of intersecting confidence intervals in Eq. \eqref{eq_intersecting_confidence_interval}.

To prove $\mathcal{I}_{y_{(m)}, t}  \subseteq \mathcal{I}_{y_{(m)}, t^\prime}$, we show $L_{y_{(m)}, t} \geq L_{y_{(m)}, t^\prime}$ here and the other side is similar: Since $L_{y_{(m)}, t} = {\median}\{ L_{y_i, t}\}$ and $L_{y_{(m)}, t^\prime} = {\median}\{ L_{y_i, t^\prime}\}$ and $L_{y_i, t} \geq L_{y_i, t^\prime}$, invoking Lemma \ref{lm_loc_comparison} with $j = m$ lead to the desired result.

To prove $\mathcal{I}_{\AD_i, t}  \subseteq \mathcal{I}_{\AD_i, t^\prime}$, we notice that $[L_{y_i, t}, U_{y_i, t}] \subseteq [L_{y_i, t^\prime}, U_{y_i, t^\prime}]$ and $[L_{y_{(m)}, t}, U_{y_{(m)}, t}] \subseteq [L_{y_{(m)}, t^\prime}, U_{y_{(m)}, t^\prime}]$. Thus, invoking Lemma \ref{lm_FindAD} twice we have $\mathcal{I}_{\AD_i, t}  = \findad([L_{y_i, t}, U_{y_i, t}], [L_{y_{(m)}, t}, U_{y_{(m)}, t}]) \subseteq \findad([L_{y_i, t^\prime}, U_{y_i, t^\prime}],[L_{y_{(m)}, t^\prime}, U_{y_{(m)}, t^\prime}]) = \mathcal{I}_{\AD_i, t^\prime}$.

The proof of $\mathcal{I}_{\AD_{(m)}, t}  \subseteq \mathcal{I}_{\AD_{(m)}, t^\prime}$ is similar to the proof of $\mathcal{I}_{y_{(m)}, t}  \subseteq \mathcal{I}_{y_{(m)}, t^\prime}$ after noticing $\mathcal{I}_{\AD_i, t}  \subseteq \mathcal{I}_{\AD_i, t^\prime}$; the proof of $\mathcal{I}_{\theta, t}  \subseteq \mathcal{I}_{\theta, t^\prime}$ is a direct consequence of $\mathcal{I}_{y_{(m)}, t}  \subseteq \mathcal{I}_{y_{(m)}, t^\prime}$, $\mathcal{I}_{\AD_{(m)}, t}  \subseteq \mathcal{I}_{\AD_{(m)}, t^\prime}$ and the construction of $\mathcal{I}_{\theta, t}$ described in Algorithm \ref{algorithm_confidence interval}.
\end{proof}

\subsection{Confidence Interval of $\theta$}

\begin{lemma}
	\label{lm_extended_ci_median}
	In {\tt ROAIElim}, at time $t$, we have 
	\begin{equation*}
	\mathcal{I}_{y_{(m)}, t} \subseteq \left[ {y}_{(m)} - 2 \beta_{t}, {y}_{(m)} + 2 \beta_{t}\right].
	\end{equation*}
\end{lemma}

\begin{proof}
	Recall $\mathcal{I}_{y_{(m)}, t} = [L_{y_{(m)}, t}, U_{y_{(m)}, t}]$. We only prove here that ${y}_{(m)} - 2 \beta_{t}  \leq L_{y_{(m)}, t}$ here, and the other side could be proved similarly.

	Let 
	\begin{equation*}
	S^{\median}_{top, t} = \{ i \in [n]: L_{y_i, t} > U_{y_{(m)}, t} \}, \quad	S^{\median}_{bottom, t} = \{ i \in [n]: U_{y_i, t} < L_{y_{(m)}, t} \}.
	\end{equation*}
	Due to the application of intersecting confidence intervals, one could see that $S^{\median}_{top, t}$ represents an \emph{identified}, whether at time $t$ or a previous time step $t^\prime < t$, subset of arms with means greater than $y_{(m)}$, and thus $|S^{\median}_{top, t}| \leq m-1$; similarly, $S^{\median}_{bottom, t}$ represents the \emph{identified} subset of arms with means smaller than $y_{(m)}$, and $|S^{\median}_{bottom, t}| \leq m-1$. Since $A^{\median}_{E, t}$ in Algorithm \ref{algorithm_ROAIElim} essentially represents the subset of arms un-distinguished from the median $y_{(m)}$ \emph{up to} time $t$, we have
	\begin{equation*}
	A^{\median}_{E, t} = [n] \backslash(S^{\median}_{top, t} \cup S^{\median}_{bottom, t}).
	\end{equation*}
	
	Suppose $|S^{\median}_{top, t}|=k_{top, t} \leq m-1$. We first notice that $S^{\median}_{top, t}$ contains $k_{top, t}$ arms with lower bounds greater than $L_{y_{(m)}, t}$: for any $i_a \in S^{\median}_{top, t}$, there exists $t^\prime \leq t$ such that $L_{y_{i_a}, t} \geq L_{y_{i_a}, t^\prime} > U_{y_{(m)}, t^\prime} \geq L_{y_{(m)}, t}$, where the first inequality comes from \cref{lm_other_intersecting_confidence_interval} and the last inequality comes from the fact that: if $U_{y_{(m)}, t^\prime} < L_{y_{(m)}, t}$, we have $\mathcal{I}_{y_{(m)}, t^\prime} \cap \mathcal{I}_{y_{(m)}, t} = \emptyset$ contradicting with event $W$.
	
	Suppose $|S^{\median}_{bottom, t}|=k_{bottom, t} \leq m-1$. We could similarly notice that $S^{\median}_{bottom, t}$ contains $k_{bottom, t}$ arms with lower bounds smaller than $L_{y_{(m)}, t}$: for any $i_b \in S^{\median}_{bottom, t}$, there exist $t^\prime \leq t$ such that $L_{y_{i_b}, t} \leq U_{y_{i_b}, t^\prime} < L_{y_{(m)}, t^\prime} \leq L_{y_{(m)}, t}$, where the first inequality comes from the fact that $\mathcal{I}_{y_{i_b}, t} \cap \mathcal{I}_{y_{i_b}, t^\prime} \neq \emptyset$ and the last inequality comes from \cref{lm_other_intersecting_confidence_interval}.
	
	Thus, to identify $L_{y_{(m)}, t}$, we only need to identify the $(m-k_{top, t})$-th largest lower bound among arms in $A^{\median}_{E, t}$. For $i \in A^{\median}_{E, t}$, we have 
	\begin{equation*}
	y_i - 2 \beta_{t} = y_i - 2 \beta_{N_{i, t}} \leq L_{y_i,t }
	\end{equation*}
	according to Lemma \ref{lm_extended_ci} and the fact that arms in $A^{\median}_{E, t}$ are pulled $t$ times as $A^{\median}_{E, t} \subseteq A_{E, t}$.
	
	Invoking Lemma \ref{lm_loc_comparison} with $j = m-k_{top, t}$ concludes the proof as $y_{(m)} - 2 \beta_t$ is the $(m-k_{top, t})$-th largest quantity among $\{y_i - 2 \beta_{t}\}$ for $i \in A^{\median}_{E, t}$.
\end{proof}

\begin{lemma}
	\label{lm_extended_ci_AD}
	In {\tt ROAIElim}, at time $t$, for $\forall i \in A^{\MAD}_{E, t}$, we have 
	\begin{equation*}
	\mathcal{I}_{\AD_{i}, t} \subseteq \left[ {\AD_{i}} - 4 \beta_{t}, {\AD_{i}} + 4 \beta_{t} \right].
	\end{equation*}
\end{lemma}
\begin{proof}
	We are first going to quantify the \emph{extended} confidence interval of $\AD_{i}$, namely, $[\tilde{L}_{\AD_i, t}, \tilde{U}_{\AD_i, t}] \supseteq [L_{\AD_i, t}, U_{\AD_i, t}] $. 
	
	According to Lemma \ref{lm_extended_ci_median} and Lemma \ref{lm_extended_ci}, we have $\left[ {y}_{(m)} - 2 \beta_{t}, {y}_{(m)} + 2 \beta_t\right] \supseteq \mathcal{I}_{y_{(m)}, t}$ and $[y_i - 2 \beta_{N_{i, t}}, y_i + 2 \beta_{N_{i, t}}] \supseteq \mathcal{I}_{y_i, t}$. Feeding $\left[ {y}_{(m)} - 2 \beta_{t}, {y}_{(m)} + 2 \beta_t\right]$ and $[y_i - 2 \beta_{N_{i, t}}, y_i + 2 \beta_{N_{i, t}}]$ into the ${\findad}$ operator and apply Lemma \ref{lm_FindAD} twice leads to
	\begin{equation*}
	[\tilde{L}_{\AD_i, t}, \tilde{U}_{\AD_i, t}] :=[\AD_i - 2 \beta_t - 2 \beta_{N_{i, t}} , \AD_i + 2 \beta_t + 2 \beta_{N_{i, t}}]  \supseteq	[L_{\AD_i, t}, U_{\AD_i, t}].
	\end{equation*}
	We conclude the proof with $\beta_{N_{i, t}} = \beta_t$ for $\forall i \in A^{\MAD}_{E, t}$.
\end{proof}

\begin{lemma}
	\label{lm_extended_ci_MAD}
	In {\tt ROAIElim}, at time $t$, we have 
	\begin{equation*}
	\mathcal{I}_{\AD_{(m)}, t} \subseteq \left[ {\AD_{(m)}} - 4 \beta_{t}, {\AD_{(m)}} + 4 \beta_{t} \right].
	\end{equation*}
\end{lemma}

\begin{proof}
	The proof of this Lemma largely follows from the proof of Lemma \ref{lm_extended_ci_median}. As before, we will give a proof for ${\AD_{(m)}} - 4 \beta_{t} \leq L_{\AD_{(m)}, t}$, and the other side could be proved similarly.
	
	Let
	\begin{equation*}
	S^{\MAD}_{top, t} = \{ i \in [n]: L_{\AD_i, t} > U_{\AD_{(m)}, t} \}, \quad	S^{\MAD}_{bottom, t} = \{ i \in [n]: U_{\AD_i, t} < L_{\AD_{(m)}, t} \}.
	\end{equation*}
	
	According to Lemma \ref{lm_other_intersecting_confidence_interval}, one could see that $S^{\MAD}_{top, t}$ represents an \emph{identified}, whether at time $t$ or a previous time step $t^\prime < t$, subset of arms with absolute deviations greater than $\AD_{(m)}$, and thus $|S^{\MAD}_{top, t}| \leq m-1$; similarly, $S^{\MAD}_{bottom, t}$ represents the \emph{identified} subset of arms with absolute deviations smaller than $\AD_{(m)}$, and $|S^{\MAD}_{bottom, t}| \leq m-1$. Since $A^{\MAD}_{E, t}$ in Algorithm \ref{algorithm_ROAIElim} essentially represents the subset of arms un-distinguished from the median absolute deviation $\AD_{(m)}$ \emph{up to} time $t$, we have
	\begin{equation*}
	A^{\MAD}_{E, t} = [n] \backslash(S^{\MAD}_{top, t} \cup S^{\MAD}_{bottom, t}).
	\end{equation*}

	Suppose $|S^{\MAD}_{top, t}|=k_{top, t} \leq m-1$. We first notice that $S^{\MAD}_{top, t}$ contains $k_{top, t}$ arms with lower bounds on absolute deviation greater than $L_{\AD_{(m)}, t}$: for any $i_a \in S^{\MAD}_{top, t}$, there exists $t^\prime \leq t$ such that $L_{\AD_{i_a}, t} \geq L_{\AD_{i_a}, t^\prime} > U_{\AD_{(m)}, t^\prime} \geq L_{\AD_{(m)}, t}$, where the first inequality comes from Lemma \ref{lm_other_intersecting_confidence_interval} and the last inequality come from the fact that: if $U_{\AD_{(m)}, t^\prime} < L_{\AD_{(m)}, t}$, we have $\mathcal{I}_{\AD_{(m)}, t^\prime} \cap \mathcal{I}_{\AD_{(m)}, t} = \emptyset$ contradicting with event $W$.
	
	Suppose $|S^{\MAD}_{bottom, t}|=k_{bottom, t} \leq m-1$. We could similarly notice that $S^{\MAD}_{bottom, t}$ contains $k_{bottom, t}$ arms with lower bounds smaller than $L_{\AD_{(m)}, t}$: for any $i_b \in S^{\MAD}_{bottom, t}$, there exist $t^\prime \leq t$ such that $L_{\AD_{i_b}, t} \leq U_{\AD_{i_b}, t^\prime} < L_{\AD_{(m)}, t^\prime} \leq L_{\AD_{(m)}, t}$, where the first inequality comes from the fact that $\mathcal{I}_{\AD_{i_b}, t} \cap \mathcal{I}_{\AD_{i_b}, t^\prime} \neq \emptyset$ and the last inequality comes from Lemma \ref{lm_other_intersecting_confidence_interval}.
	
	Thus, to identify $L_{\AD_{(m)}, t}$, we only need to identify the $(m-k_{top, t})$-th largest lower bound among arms in $A^{\MAD}_{E, t}$. For $i \in A^{\MAD}_{E, t}$, according to Lemma \ref{lm_extended_ci_AD}, we have 
	\begin{equation*}
    \AD_i - 4 \beta_t  = \AD_i - 2 \beta_t - 2 \beta_{N_{i, t}} \leq L_{\AD_{i}, t}.
    \end{equation*}
    Invoking Lemma \ref{lm_loc_comparison} with $j = m-k_t$ and with respect to $i \in  A^{\MAD}_{E, t}$ leads to the desired result as $\AD_{(m)} - 4 \beta_t$ is the $(m - k_{top, t})$-th largest quantity among $\{\AD_i - 4 \beta_t\}$ for $i \in A^{\MAD}_{E, t}$.
\end{proof}

\begin{lemma}
	\label{lm_extended_ci_threshold}
	In {\tt ROAIElim}, at time $t$, we have 
	\begin{equation*}
	\mathcal{I}_{\theta, t} \subseteq \left[ \theta - (2 + 4 k) \beta_{t}, \theta + (2 + 4 k) \beta_{t} \right].
	\end{equation*}
\end{lemma}
\begin{proof}
Combining Lemma \ref{lm_extended_ci_median} and \ref{lm_extended_ci_MAD} with Eq. \eqref{eq_robust_threshold} immediately gives the desired result.
\end{proof}

\subsection{Sample complexity analysis}

The sample complexity of \roaielim could be characterized in the following theorem.

\begin{theorem}
	\label{thm_complexity_Elim}
	With probability of at least $1 - \delta$, the sample complexity of \roaielim is upper bounded by 
	\begin{equation*}
O \left(  \sum_{i=1}^n   \frac{ \log \left( n / \delta \tilde\Delta_i^{*} \right)}{(\tilde\Delta_i^{*})^2}\right),
\end{equation*}
		where	
	\begin{equation*}
	\tilde{\Delta}^*_i = \max \{\Delta_*^{\theta}/(1+k),  \min \{ \Delta_i^{\theta}/(1+k), \Delta_i^{\median}, \Delta_i^{\MAD} \} \}.
	\end{equation*}
\end{theorem}
\begin{proof}
    \roaielim stops sampling an arm $i$ if either the algorithm stops, or if arm $i$ is eliminated from the active set $A_{E,t}$. We analyze the sample complexity of both these events. The number of samples of arm i is then the minimum of these two sample complexities.
	
	\roaielim stops when $A^{\theta}_{E, t} = \emptyset$. When $A^{\theta}_{E, t}  \neq \emptyset$, for all $i \in A^{\theta}_{E, t} $, we have $\beta_{N_{i, t}} = \beta_t$. Notice that, based on Lemma \ref{lm_extended_ci} and \ref{lm_extended_ci_threshold}, when $(2 + 4 k) \beta_{t} + 2 \beta_{t} < \Delta_i^{\theta} $, $i \notin A^{\theta}_{E, t} $. Remark \ref{remark_sample_needed} immediately shows that there exists a constant $C$ such that when $N_{i, t} \geq C \frac{ (1+k)^2\log \left( n (1+k) / \delta \Delta_i^{\theta} \right)}{(\Delta_i^{\theta})^2}$, arm $i$ is guaranteed to be expelled from $A^{\theta}_{E, t} $. As another consequence, no arm will be pulled more than $C \frac{ (1+k)^2\log \left( n (1+k) / \delta \Delta_*^{\theta} \right)}{(\Delta_*^{\theta})^2}$ times as that's when the algorithm stops, i.e., $A^{\theta}_{E, t} = \emptyset$.
	
    We next calculate the number of samples before $i \notin A_{E,t}$. Note that $A_{E, t} = A^{\median}_{E, t} \cup A^{\MAD}_{E, t} \cup A^{\theta}_{E, t}$, thus we only need to further consider when arm $i$ is out of set $A^{\median}_{E, t}$ and set $A^{\MAD}_{E, t}$. Based on Lemma \ref{lm_extended_ci} and \ref{lm_extended_ci_median}, we know that when $2\beta_t + 2\beta_t < \Delta^{\median}_i$, we have $i \notin A^{\median}_{E, t}$. Remark \ref{remark_sample_needed} shows that there exists a constant $C$ such that when $N_{i, t}  \geq C \frac{\log (n / \delta \Delta_i^{\median})}{(\Delta_i^{\median})^2}$, we have $2\beta_t + 2\beta_t < \Delta^{\median}_i $, i.e., $i \notin A^{\median}_{E, t}$. In a similar manner, invoking Lemma \ref{lm_extended_ci_AD} and Lemma \ref{lm_extended_ci_MAD}, we have $i \notin A^{\MAD}_{E, t}$ when $N_{i, t} \geq C \frac{\log (n / \delta \Delta_i^{\MAD})}{(\Delta_i^{\MAD})^2}$.
	
	To summarize, the total number of pulls on arm $i \in [n]$ could be upper bounded by 
	\begin{equation*}
	O \left(  \sum_{i=1}^n   \frac{ \log \left( n / \delta \tilde\Delta_i^{*} \right)}{(\tilde\Delta_i^{*})^2}\right).
	\end{equation*}

\end{proof}

\section{Sample Complexity Analysis of \roailucb}
\label{appendix_ROAILucb}

We first prove some supporting Lemmas in Appendix \ref{appendix_ROAILucb_supporting}, \ref{appendix_ROAILucb_outlier}, \ref{appendix_ROAILucb_median} and \ref{appendix_ROAILucb_MAD}, and then move to the proof of sample complexity in Appendix \ref{appendix_ROAILucb_sample_complexity}. As before, we will prove sample complexity upper bound under the good event $W$, which happens with probability of at least $1-\delta$. During the analysis, for any quantity indexed by two arguments $q$ and $t$ (time step), if $q$ itself already indicates the time step, we will simply drop the second argument $t$. For example, we will simplify $L_{y_{l_{\theta, t}}, t}$ as $L_{y_{l_{\theta, t}}}$.

\subsection{Supporting Lemma and Notation}
\label{appendix_ROAILucb_supporting}

Recall that we assume the number of arms $n = 2m - 1$ is odd, $y_i \geq y_{i+1}$ and the set of outlier arms is $S_o = \{1, \dots, n_1\}$. In the identifiable case, i.e., $y_i \neq \theta$, we then know $\theta \in (y_{n_1}, y_{n_1 + 1})$. 

\begin{lemma}
\label{lm_length_confidence_bound_AD}
For any $i \in [n]$ and $t \in \mathbb{N}$, we have ${\len}(\mathcal{I}_{\AD_i, t}) \leq {\len}(\mathcal{I}_{y_i, t}) + {\len}(\mathcal{I}_{y_{(m)}, t})$.
\end{lemma}
\begin{proof}
Suppose $\mathcal{I}_{y_i, t} = [L_{y_i, t}, U_{y_i, t}] = [\hat{y}_{i, t} - \beta_{N_{i, t}}, \hat{y}_{i, t} + \beta_{N_{i, t}}]$, and $\mathcal{I}_{y_{(m)}, t} = [L_{y_{(m)}, t}, U_{y_{(m)}, t}] = [\tilde{y}_{t} - \tilde{\beta}_t, \tilde{y}_{t} + \tilde{\beta}_t]$ with $\tilde{y}_{t}$ defined as the midpoint of $[L_{y_{(m)}, t}, U_{y_{(m)}, t}]$ and $\tilde{\beta}_t = {\len}(\mathcal{I}_{y_{(m), t}})/2$. Let $\mathcal{I}_{\AD_{i, t}} = [L_{\AD_{i}, t}, U_{\AD_{i}, t}]$, by construction of $\mathcal{I}_{\AD_{i}, t}$ in Algorithm \ref{algorithm_confidence interval},\footnote{We consider the version with slight modification $L_{\AD_i, t} = \max \{0,  L_{y_i, t} - U_{y_{(m)}, t}, L_{y_{(m)}, t} - U_{y_i, t}  \}$; the reasoning is mentioned in the proof of \cref{lem:confidence_interval}.} we have 
\begin{equation*}
U_{\AD_i, t} = \max \{ U_{y_i, t} - L_{y_{(m)}, t}, U_{y_{(m)}, t} - L_{y_i, t} \} = |\hat{y}_{i, t} - \tilde{y}_{t}| + \beta_{N_{i, t}} + \tilde{\beta}_t.
\end{equation*}
On the other side, we have 
\begin{equation*}
L_{\AD_i, t} = \max \{ 0, L_{y_i, t} - U_{y_{(m)}, t}, L_{y_{(m)}, t} - U_{y_i, t}  \} = \max \{ 0,  |\hat{y}_{i, t} - \tilde{y}_{t}| - \beta_{N_{i, t}} - \tilde{\beta}_t\} \geq |\hat{y}_{i, t} - \tilde{y}_{t}| - \beta_{N_{i, t}} - \tilde{\beta}_t,
\end{equation*}
where the second inequality comes from the fact that $\max \{ L_{y_i, t} - U_{y_{(m)}, t}, L_{y_{(m)}, t} - U_{y_i, t} \}= |\hat{y}_{i, t} - \tilde{y}_{t}| - \beta_{N_{i, t}} - \tilde{\beta}_t$.

Thus, we have ${\len}(\mathcal{I}_{\AD_i, t})  = U_{\AD_i, t}  - L_{\AD_i, t}  \leq 2 \beta_{N_{i, t}} + 2 \tilde{\beta}_t = \len(\mathcal{I}_{y_i, t}) + \len(\mathcal{I}_{y_{(m)}, t})$.
\end{proof}

We next define the following set of constants, which we shall refer to frequently in our analysis. 

\begin{equation}
\label{eq_ROAILucb_constant}
\begin{cases}
c^{\theta}_1 =   \frac{y_{n_1} + \theta}{2}, \\
c^{\theta}_2 =   \frac{y_{n_1 + 1} + \theta}{2},\\
c^{\median}_{1} = \frac{y_{(m-1)} + y_{(m)}}{2},\\
c^{\median}_{2} = \frac{y_{(m)}+ y_{(m+1)}}{2},\\
c^{\MAD}_{1} = \frac{AD_{(m-1)} + AD_{(m)}}{2},\\
c^{\MAD}_{2} = \frac{AD_{(m)} + AD_{(m+1)}}{2}.
\end{cases}
\end{equation}
For $i \in \{1,2\}$, we could potentially have $c^{\median}_i = y_{(m)}$ or $c^{\MAD}_i = \AD_{(m)}$ when there exists multiple medians among $\{y_i\}$ or $\{\AD_i\}$; we have $c^{\theta}_i \neq \theta$ due to assumption on identifiability. Note that these constants are only used in analysis, and \cref{algorithm_ROAILucb} proceeds without knowing constants defined in Eq. \eqref{eq_ROAILucb_constant}.

\subsection{Analysis for Outlier Identification}
\label{appendix_ROAILucb_outlier}
In this section, we define $\NEEDY^{\theta}_{i, t}$, which denotes the event arm $i$/threshold $\theta$ is needy at round $t$ in the sense of determining outlier/normal arms, and analyze its properties.

\begin{definition}
	At any time $t$, we separate arms into three different subsets as follows, according to the relation of their confidence intervals and unknown constants $c^{\theta}_i$,
	\begin{align*}
	S^{\theta}_{top, t} & = \left\{ i \in [n] : L_{y_i, t} > c^{\theta}_{1} \right\}, \\
	S^{\theta}_{bottom, t} & = \left\{ i \in [n]  : U_{y_i, t} < c^{\theta}_{2} \right\}, \\
	S^{\theta}_{middle, t} & = \left\{ i \in [n]: i \notin S^{\theta}_{top, t}  \cup S^{\theta}_{bottom, t} \right\}. 
	\end{align*}
		Then, we define the following $\NEEDY$ event for arm $i \in [n]$ in the sense of to be separated from $\theta$
	\begin{equation*}
	\NEEDY^{\theta}_{i, t} =  \left( i \in S^{\theta}_{middle, t}  \right),
	\end{equation*}
	we also define the $\NEEDY$ event for the outlier threshold $\theta$
	\begin{equation*}
	\NEEDY^{\theta}_{\theta, t} =  \left( c^{\theta}_1 \in \mathcal{I}_{\theta, t} \right) \cup \left( c^{\theta}_2 \in \mathcal{I}_{\theta, t}  \right).
	\end{equation*}
\end{definition}
Recall $\hat{S}_{o, t} = \{ i \in [n]: \hat{y}_{i, t} > \hat{\theta}_t\}$ and $\hat{S}_{n, t} = [n]\backslash \hat{S}_{o, t}$. We define (breaking ties arbitrarily)
\begin{equation*}
l_{{\theta}, t} = \argmin_{a \in \hat{S}_{o, t}} \{L_{y_a, t} \}, \quad  u_{{\theta}, t} = \argmax_{a \in \hat{S}_{n, t}}  \{U_{y_a, t} \}.
\end{equation*}

\begin{lemma}
	\label{lm_non_termination_global}
	If Algorithm \ref{algorithm_ROAILucb} doesn't stop at time $t$, then there exists $a \in  \{l_{\theta, t}, u_{\theta, t}, \theta\}$ such that $\NEEDY^{\theta}_{a, t}$ holds.
\end{lemma}
\begin{proof}
	We analyze this under the good event $W$ where all confidence intervals are valid. If $\NEEDY^{\theta}_{a, t}$ don't occur for all $a \in  \{l_{\theta, t}, u_{\theta, t}, \theta\}$, we then show Algorithm \ref{algorithm_ROAILucb} necessarily terminates as followings. 
	
	We first notice that when $\NEEDY^\theta_{\theta, t}$ doesn't occur, we will have $\mathcal{I}_{\theta, t} \subseteq (c^{\theta}_2, c^{\theta}_1)$ according to the fact that $\theta \in \mathcal{I}_{\theta, t}$ and the definition of $c^{\theta}_1$ and $c^{\theta}_2$. Secondly, if $\NEEDY^{\theta}_{l_{\theta, t}}$ doesn't occur either, we necessarily have $L_{y_{l_{\theta, t}}} > c^{\theta}_1$ as $U_{y_{l_{\theta, t}}} < c^{\theta}_2$ cannot be true due to $\hat{y}_{l_{\theta, t}} > \hat{\theta}_t$; as a consequence, for any $i \in \hat{S}_{o, t}$, we have $L_{y_i, t} \geq L_{y_{l_{\theta, t}}} > c_1^{\theta}$. Similarly, we have $U_{y_i, t} < c^{\theta}_2$ for any $i \in \hat{S}_{n, t}$ if $\NEEDY^{\theta}_{u_{\theta, t}}$ doesn't occur either. To summarize, we have $A^{\theta}_{L, t} = \emptyset$, which indicates the termination of Algorithm \ref{algorithm_ROAILucb}.
\end{proof}

\begin{lemma}
	\label{lm_non_termination_local}
	If $\NEEDY^{\theta}_{\theta, t} $ holds, then we have $\len(\mathcal{I}_{\theta, t}) \geq {\Delta_{*}^{\theta}}/{2}$; furthermore, we have either $\len(\mathcal{I}_{y_{(m)}, t}) \geq \epsilon^{\median}$ or $\len(\mathcal{I}_{AD_{(m)}, t}) \geq \epsilon^{\MAD}$ with $\epsilon^{\median} := {\Delta_{*}^\theta }/{(2 + 8k)}$ and $\epsilon^{\MAD} := {\Delta_{*}^\theta }/{(1/2 + 2 k)}$.
\end{lemma}
\begin{proof}
	Notice that $\min\{ c^{\theta}_1 - \theta, \theta - c^{\theta}_2\} \geq \Delta^{\theta}_*/2$, thus the first statement is necessarily true to ensure $\theta \in \mathcal{I}_{\theta, t}$; the second statement need to be true as otherwise we will have $\len(\mathcal{I}_{\theta, t}) < \Delta_{*}^{\theta} / 2$ due to carefully chosen $\epsilon^{\median}$ and $\epsilon^{\MAD}$ such that $\epsilon^{\median} + k \cdot \epsilon^{\MAD} = \Delta_{*}^{\theta} /2$.
\end{proof}

\begin{remark}
	\label{rm_twp_eps_relation}
	Here we deliberately chose $\epsilon^{\median} = {\epsilon^{\MAD}}/{4}$ here for convenience, while our analysis works as long as $\epsilon^{\median} < {\epsilon^{\MAD}}$. One may also optimize over $\epsilon^{\median}$ and $\epsilon^{\MAD}$ to get a slightly tighter, in terms of constant, sample complexity upper bound.
\end{remark}

\subsection{Analysis for Median Identification}
\label{appendix_ROAILucb_median}
In this section, we define $\NEEDY^{\median}_{i, t}$, which denotes the event arm $i$ is needy at round $t$ in the sense of shrinking the confidence interval of $y_{(m)}$, and analyze its properties.

\begin{definition}
	\label{def_ROAILucb_median}
		At any time $t$, we separate arms into three different subsets as follows, according to the relation of their confidence intervals and unknown constants $c^{\median}_i$,
		\begin{align*}
			S^{\median}_{top, t} & = \left\{ i \in [n]: L_{y_i, t} > c^{\median}_{1} \right\}, \\
		S^{\median}_{bottom, t}& = \left\{ i \in [n]: U_{y_i, t} < c^{\median}_{2} \right\}, \\
		S^{\median}_{middle, t} & = \left\{ i \in [n]: i \notin 	S^{\median}_{top, t} \cup	S^{\median}_{bottom, t}\right\}. 
		\end{align*}
		Then, we define the following $\NEEDY$ event for arm $i \in [n]$ in the sense of shrinking the confidence interval of $y_{(m)}$
	\begin{equation*}
	\NEEDY^{\median}_{i, t} =  \left( i \in 	S^{\median}_{middle, t} \right)  \cap \left(\beta_{N_{i, t}} \geq {\epsilon^{\median}}/{2} \right).
	\end{equation*}
\end{definition}

\begin{remark}
	\label{remark_needy_median}
	Note that $S^{\median}_{middle, t}$ is equivalent to $\left\{ i \in [n]: c^{\median}_j \in \mathcal{I}_{y_i, t} \right\}  \cup  \left\{ i \in [n]: y_i = y_{(m)}\right\}$ as arms in $\left\{ i \in [n]: y_i = y_{(m)}\right\}$ cannot be in $S^{\median}_{top, t}$ or $S^{\median}_{bottom, t}$ under the good event $W$.
\end{remark}

We perform LUCB at both $(m-1)$-th and $m$-th locations, and aim at shrinking the confidence interval of the median below length $\epsilon^{\median}$, i.e., $\len(\mathcal{I}_{y_{(m)}, t}) < \epsilon^{\median}$. Recall we set $\kappa_1 = m-1$ and $\kappa_2 = m$; and for $i \in \{1, 2\}$, we let $J_{\kappa_i, t}$ denote a subset of $\kappa_i$ arms with the highest empirical rewards among $\{\hat{y}_i\}$, breaking ties arbitrarily. For $i = {1, 2}$, we further define
\begin{equation}
\label{eq_critical_arm_median}
l_{i, t} = \argmin_{a \in J_{\kappa_i, t}}  \{L_{y_a, t} \} , \quad  u_{i, t} = \argmax_{a \notin J_{\kappa_i, t}}  \{ U_{y_a, t} \}
\end{equation}
to be the two critical arms from $J_{\kappa_i, t}$ and $(J_{\kappa_i, t})^c$ that are likely to be misclassified. Recall that, to simplify notations, we will ignore the second subscript $t$ whenever $u_{i, t}$ or $l_{i, t}$ appears in the first subscript, which already indicates the time step $t$.

\begin{lemma}
	\label{lm_median_needy_1}
	For any $i \in \{1, 2\}$, if $U_{y_{u_{i, t}}} - L_{y_{l_{i, t}}} \geq \epsilon^{\median}$ holds, then $\NEEDY^{\median}_{k, t}$ holds for either $k = l_{i, t}$ or $k = u_{i, t}$, i.e., $k$ satisfies 
	\begin{equation}
	\label{eq_location_not_separated}
	c^{\median}_i \in \mathcal{I}_{y_k, t}, \ \text{and} \  \beta_{N_{k, t}} \geq {\epsilon^{\median}}/{2}.
	\end{equation}
\end{lemma}
\begin{proof}
	The main idea of this proof comes from \cite{kalyanakrishnan2012pac}; we provide the proof here for completeness.
	
	We start arguing that $c^{\median}_i \in \mathcal{I}_{y_{u_{i, t}}}$ or $ c^{\median}_i \in \mathcal{I}_{y_{l_{i, t}}}$ by arguing the following four exclusive cases cannot be true under event $W$.
	
	\textbf{Case 1. $c^{\median}_i > U_{y_{u_{i, t}}}$ and $c^{\median}_i > U_{y_{l_{i, t}}}$:} This indicates there will be at least $n-\kappa_i + 1$ arms with means smaller than $c^{\median}_i$ as all $n-\kappa_i$ arms in $(J_{\kappa_i, t})^c$ have upper bounds smaller than $c^{\median}_i$ and at least one arm in $J_{\kappa_i, t}$, i.e., arm $L_{y_i, t}$, has upper bound smaller than $c^{\median}_i$; on the other side, we can have at most $n-\kappa_i$ arms with means smaller than $c^{\median}_i$ according to definition in Eq. \eqref{eq_ROAILucb_constant}, which leads to a contradiction.
	
	\textbf{Case 2. $c^{\median}_i > U_{y_{u_{i, t}}}$ and $c^{\median}_i < L_{y_{l_{i, t}}}$:} This indicates $U_{y_{u_{i, t}}} < L_{y_{l_{i, t}}}$, which contradicts with the fact that $U_{y_{u_{i, t}}} - L_{y_{l_{i, t}}} \geq \epsilon^{\median} > 0$.
	
	\textbf{Case 3. $c^{\median}_i < L_{y_{u_{i, t}}}$ and $c^{\median}_i > U_{y_{l_{i, t}}}$:} This leads to the contradiction that $c^{\median}_i < L_{y_{u_{i, t}}} \leq \hat{y}_{u_{i, t}} \leq \hat{y}_{l_{i, t}} \leq U_{y_{l_{i, t}}} < c^{\median}_i$, where the third inequality comes from the fact ${u_{i, t}} \notin J_{\kappa_i, t}$ and ${l_{i, t}} \in J_{\kappa_i, t}$.
	
	\textbf{Case 4. $c^{\median}_i < L_{y_{u_{i, t}}}$ and $c^{\median}_i < L_{y_{l_{i, t}}}$:} Similar to Case 1, Case 4 indicates there will be at least $\kappa_i + 1$ arms with mean greater than $c^{\median}_i$, which contradicts with the fact that there can have at most $\kappa_i$ such arms.

We next show Eq. \eqref{eq_location_not_separated} holds true by considering two situations: (1) $c^{\median}_i$ belongs to both $\mathcal{I}_{y_{u_{i, t}}}$ and $\mathcal{I}_{y_{l_{i, t}}}$; (2) $c^{\median}_i$ only belongs to one of $\mathcal{I}_{y_{u_{i, t}}}$ and $\mathcal{I}_{y_{l_{i, t}}}$.
	
	 In both situations, we notice that 
	\begin{equation}
	\label{eq_needy_sum_ci}
	\beta_{N_{u_{i, t}}} + \beta_{N_{l_{i, t}}} \geq \hat{y}_{u_{i, t}}  + \beta_{N_{u_{i, t}}} - (\hat{y}_{l_{i, t}} - \beta_{N_{l_{i, t}}}) = U_{y_{u_{i, t}}} - L_{y_{l_{i, t}}} \geq \epsilon^{\median},
	\end{equation}
	where the first inequality comes from $\hat{y}_{u_{i, t}} \leq \hat{y}_{l_{i, t}}$.

	In the first situation: Since $c^{\median}_i \in \mathcal{I}_{y_{u_{i, t}}}$ and $c^{\median}_i \in \mathcal{I}_{y_{l_{i, t}}}$, and we also have either $\beta_{N_{u_{i, t}}} \geq {\epsilon^{\median}}/{2}$ or $\beta_{N_{l_{i, t}}} \geq {\epsilon^{\median}}/{2}$; thus Eq. \eqref{eq_location_not_separated} is satisfied.
	
	In the second situation: We consider when $c^{\median}_i$ only belongs to one of the confidence intervals. Specifically, we consider the following four exclusive cases:
	
    \textbf{Case 1. $c^{\median}_i \in \mathcal{I}_{y_{u_{i, t}}}, \ \text{and} \ \ c^{\median}_i > U_{y_{l_{i, t}}}  \implies \beta_{N_{u_{i, t}}} \geq {\epsilon^{\median}}/{2}$:}
    
    This case indicates $\hat{y}_{u_{i, t}} + \beta_{N_{u_{i, t}}} \geq c_i^{\median} > \hat{y}_{l_{i, t}} + \beta_{N_{l_{i, t}}}$; combine this with the fact that $\hat{y}_{u_{i, t}} \leq \hat{y}_{l_{i, t}}$, we further have
    \begin{equation}
    \label{eq_needy_case1_larger_ci}
    \beta_{N_{u_{i, t}}} \geq \beta_{N_{l_{i, t}}}.
    \end{equation}
    Combine Eq. \eqref{eq_needy_case1_larger_ci} with Eq. \eqref{eq_needy_sum_ci} leads to the desired result.
    
    \textbf{Case 2. $c^{\median}_i \in \mathcal{I}_{y_{u_{i, t}}}, \ \text{and} \ \ c^{\median}_i < L_{y_{l_{i, t}}}  \implies \beta_{N_{u_{i, t}}} \geq {\epsilon^{\median}}/{2}$:}
    
    This case indicates $\hat{y}_{u_i, t} - \beta_{N_{u_{i, t}}} \leq c^{\median}_i$; combine this with the fact that
    	\begin{equation*}
    \hat{y}_{u_i, t} + \beta_{N_{u_{i, t}}} \geq \hat{y}_{l_{i, t}} - \beta_{N_{l_{i, t}}} + \epsilon^{\median} = L_{y_{l_{i, t}}} + \epsilon^{\median} > c^{\median}_i + \epsilon^{\median}
    \end{equation*}
    leads to the desired result.

    \textbf{Case 3. $c^{\median}_i \in \mathcal{I}_{y_{l_{i, t}}}, \ \text{and} \ \ c^{\median}_i > U_{y_{u_{i, t}}}  \implies \beta_{N_{l_{i, t}}} \geq {\epsilon^{\median}}/{2}$:} The proof is similar to Case 2.
    
    \textbf{Case 4. $c^{\median}_i \in \mathcal{I}_{y_{l_{i, t}}}, \ \text{and} \ \  c^{\median}_i < L_{y_{u_{i, t}}} \implies \beta_{N_{l_{i, t}}} \geq {\epsilon^{\median}}/{2}$:} The proof is similar to Case 1.
    \end{proof}

\begin{lemma}
	\label{lm_median_needy_2}
	If $\len(\mathcal{I}_{y_{(m)}, t}) \geq \epsilon^{\median}$ and $U_{y_{u_{i, t}}} - L_{y_{l_{i, t}}} < \epsilon^{\median}$ for both $i = 1, 2$, then there exists an arm $k \in \{ l_{1, t}, l_{2, t},  u_{1, t}, u_{2, t }\}$ such that $\NEEDY^{\median}_{k, t}$ holds.
\end{lemma}

\begin{proof}
    We first notice $[L_{y_{(m)}, t} , U_{y_{(m)}, t}] \subseteq [L_{y_{l_{2, t}}}, U_{y_{u_{1, t}}}]$ based on the selection of ${l_{2, t}}$ and $u_{1, t}$ in Eq. \eqref{eq_critical_arm_median}. We then prove the Lemma by considering the following two exclusive cases:
	
	\textbf{Case 1. $U_{y_{u_{1, t}}} = U_{y_{u_{2, t}}}$ or $L_{y_{l_{1, t}}} = L_{y_{l_{2, t}}}$:} We immediately have $\len(\mathcal{I}_{y_{(m)}, t}) \leq U_{y_{u_{1, t}}} - L_{y_{l_{2, t}}} < \epsilon^{\median}$ according to $U_{y_{u_{i, t}}} - L_{y_{l_{i, t}}} < \epsilon^{\median}$ for both $i = 1, 2$; but this contradicts with the assumption $\len(\mathcal{I}_{y_{(m)}, t}) \geq \epsilon^{\median}$. Thus, this case cannot happen.
	
	\textbf{Case 2. $U_{u_{1, t}} \neq U_{u_{2, t}}$ and $L_{l_{1, t}} \neq L_{l_{2, t}}$:} Let $k$ be the index of the empirical median arm at time $t$ according to the ranking, i.e., $k = J_{\kappa_2, t} \backslash J_{\kappa_1, t}$, we then know $u_{1, t} = l_{2, t} = k$. This further leads to $\len(\mathcal{I}_{y_k, t}) \geq \len(\mathcal{I}_{y_{(m)}, t}) \geq \epsilon^{\median}$ and thus $\beta_{N_{k, t}} \geq \epsilon^{\median}/2$. We next show $\NEEDY_{k, t}^{\median}$ holds true by showing that we have $k \in S^{\median}_{middle, t}$ for either of the three sub-cases: 
	
	(1) if $y_k = y_{(m)}$, we have $k \in S^{\median}_{middle, t}$ according to Remark \ref{remark_needy_median}; 
	
	(2) if $y_k > y_{(m)}$, we know that $y_k \geq y_{(m-1)}$. Since $y_{(m)} \in \mathcal{I}_{y_{(m)}, t} \subseteq \mathcal{I}_{y_k, t}$ and $y_k \in \mathcal{I}_{y_k, t}$, we then know $c^{\median}_1 = (y_{(m)} + y_{(m-1)})/2 \in \mathcal{I}_{y_k, t}$, which leads to $k \in S^{\median}_{middle, t}$; 
	
	(3) if $y_k < y_{(m)}$, similar to sub-case (2), we have $c^{\median}_2 \in \mathcal{I}_{y_k, t}$, which leads to $k \in S^{\median}_{middle, t}$.
\end{proof}

\begin{lemma}
	\label{lm_median_needy}
	If $\len(\mathcal{I}_{y_{(m)}, t}) \geq \epsilon^{\median}$, then there exists an arm $k = l_{i, t}$ or $k = u_{i, t}$ such that $\NEEDY^{\median}_{k, t}$ holds.
\end{lemma}

\begin{proof}
	This Lemma is a direct consequence of the Lemma \ref{lm_median_needy_1} and Lemma \ref{lm_median_needy_2}.
\end{proof}

\subsection{Analysis for MAD Identification}
\label{appendix_ROAILucb_MAD}

In this section, we define $\NEEDY^{\MAD}_{i, t}$, which denotes the event arm $i$ is needy at round $t$ in the sense of shrinking the confidence interval of $\AD_{(m)}$, and analyze its properties.

\begin{definition}
	\label{def_ROAILucb_MAD}
	At any time $t$, we separate arms into three different subsets as follows, according to the relation of the confidence intervals of absolute deviations and unknown constants $c^{\MAD}_i$,
	\begin{align*}
		S^{\MAD}_{top, t} & = \left\{ i \in [n]: L_{\AD_i, t} > c^{\MAD}_{1} \right\}, \\
	S^{\MAD}_{bottom, t} & = \left\{ i \in [n]: U_{\AD_i, t} < c^{\MAD}_{2} \right\}, \\
	S^{\MAD}_{middle, t} & = \left\{ i \in [n]: i \notin S^{\MAD}_{top, t}  \cup S^{\MAD}_{bottom, t}\right\}.
	\end{align*}
	Then, we define the following $\NEEDY$ event for arm $i \in [n]$ in the sense of shrinking the confidence interval of $\AD_{(m)}$
	\begin{equation*}
	\NEEDY^{\MAD}_{i, t} =  \left(   i \in S^{\MAD}_{middle, t} \right) \cap \left(  \beta_{N_{i, t}} > {\epsilon^{\MAD}}/{4} \right).
	\end{equation*}
\end{definition}

\begin{remark}
	\label{remark_needy_MAD}
	Note that $S^{\MAD}_{middle, t}$ is equivalent to $\left\{ i \in [n]: c^{\MAD}_j \in \mathcal{I}_{\AD_i, t} \right\}  \cup \left\{ i \in [n]:  \AD_i = \AD_{(m)}\right\} $, as arms in $\left\{ i \in [n]:  \AD_i = \AD_{(m)}\right\}$ cannot be in set $S^{\MAD}_{top, t}$ or set $S^{\MAD}_{bottom, t}$ under the good event $W$. 
\end{remark}

Algorithm \ref{algorithm_ROAILucb} performs LUCB at both $(m-1)$-th and $m$-th locations with respect to $\widehat{\AD}_{i, t}$ and $\{L_{\AD_i, t}, U_{\AD_i, t}\}$, and aim at shrinking the length of $\mathcal{I}_{\AD_{(m), t}}$ below $\epsilon^{\MAD}$. Recall we set $\kappa_1 = m-1$ and $\kappa_2 = m$; and for $i \in \{1, 2\}$, we let $J^{\AD}_{\kappa_i, t}$ denote a subset of $\kappa_i$ arms with the largest empirical absolute deviations among $\{\widehat{\AD}_i\}$, breaking ties arbitrarily. For $i = {1, 2}$, we further define
\begin{equation}
\label{eq_critical_arm_MAD}
l^{\AD}_{i, t} = \argmin_{a \in J^{\AD}_{\kappa_i, t}}  \left\{L_{\AD_a, t}\right\}, \quad u^{\AD}_{i, t} = \argmax_{a \notin J^{\AD}_{\kappa_i, t}} \left\{ U_{\AD_a, t} \right\}
\end{equation}
to be the two critical arms from $J^{\AD}_{\kappa_i, t}$ and $(J^{\AD}_{\kappa_i, t})^c$ that are likely to be misclassified. Recall that, to simplify notations, we will ignore the second subscript $t$ whenever $u^{\AD}_{i, t}$ or $l^{\AD}_{i, t}$ appears in the first subscript.

\begin{lemma}
	\label{lm_MAD_needy_1}
	Assume $\len(\mathcal{I}_{y_{(m)}, t}) < \epsilon^{\median}$. If $U_{\AD_{u^{\AD}_{i, t}}} - L_{\AD_{l^{\AD}_{i, t}}} \geq \epsilon^{\MAD}$ holds, then either $k = l^{\AD}_{i, t}$ or $k = u^{\AD}_{i, t}$ and satisfies 
	\begin{equation}
	\label{eq_location_not_separated_AD}
	c^{\MAD}_i \in \mathcal{I}_{\AD_{k}, t}, \ \text{and} \  \beta_{N_{k, t}} > {\epsilon^{\MAD}}/{4}.
	\end{equation}
\end{lemma}

\begin{proof}
	Since we deliberately define select $\widehat{\AD}_i$ to be the median point of its confidence interval, i.e., $U_{\AD_{i},t} - \widehat{\AD}_{i, t} = \widehat{\AD}_{i, t} - L_{\AD_{i}, t}$, similar to the proof of Lemma \ref{lm_median_needy_1},\footnote{Note that in Lemma \ref{lm_median_needy_1}, in terms of the relation between confidence interval and the empirical value, we only use the property that $\hat{y}_{i, t}$ is the median point of the $\mathcal{I}_{y_i, t}$.} we have either $k = l^{\AD}_{i, t}$ or $k = u^{\AD}_{i, t}$ satisfies 
	\begin{equation*}
	c^{\MAD}_i \in \mathcal{I}_{\AD_{k}, t}, \ \text{and} \  \len(\mathcal{I}_{\AD_k, t}) \geq \epsilon^{\MAD}.
	\end{equation*}
	By assumption $\len(\mathcal{I}_{y_{(m)}, t}) < \epsilon^{\median} = {\epsilon^{\MAD}}/{4}$, we further obtain the following equation after invoking Lemma \ref{lm_length_confidence_bound_AD} 
	\begin{equation*}
	c^{\MAD}_i \in \mathcal{I}_{\AD_{k}, t}, \ \text{and} \ \ \beta_{N_{k, t}} > 3{\epsilon^{\MAD}}/{4}> {\epsilon^{\MAD}}/{4}.
	\end{equation*}
\end{proof}

\begin{lemma}
	\label{lm_MAD_needy_2}
	Assume $\len(\mathcal{I}_{y_{(m)}, t}) < \epsilon^{\median}$. If $\len(\mathcal{I}_{\AD_{(m)}, t}) \geq \epsilon^{\MAD}$ and $U_{\AD_{u^{\AD}_{i, t}}} - L_{\AD_{l^{\AD}_{i, t}}} < \epsilon^{\MAD}$ for both $i \in \{1, 2\}$, then there exists an arm $k \in \{l^{\AD}_{1, t}, l^{\AD}_{2, t}, u^{\AD}_{1, t}, u^{\AD}_{2, t} \}$ such that $\NEEDY^{\MAD}_{k, t}$ holds.
\end{lemma}

\begin{proof}
	We first notice $[L_{\AD_{(m)}, t} , U_{\AD_{(m)}, t}] \subseteq [L_{\AD_{l^{\AD}_{2, t}}}, U_{\AD_{u^{\AD}_{1, t}}}]$ based on the selection of ${l^{\AD}_{2, t}}$ and $u^{\AD}_{1, t}$ in Eq. \eqref{eq_critical_arm_MAD}. We then prove the Lemma by considering the following two exclusive cases:
	
	\textbf{Case 1. $U_{\AD_{u^{\AD}_{1, t}}} = U_{\AD_{u^{\AD}_{2, t}}}$ or $L_{\AD_{l^{\AD}_{1, t}}} = L_{\AD_{l^{\AD}_{2, t}}}$:} We immediately have $\len(\mathcal{I}_{\AD_{(m)}, t}) \leq U_{\AD_{u^{\AD}_{1, t}}} - L_{\AD_{l^{\AD}_{2, t}}} < \epsilon^{\MAD}$ according to $U_{\AD_{u^{\AD}_{i, t}}} - L_{\AD_{l^{\AD}_{i, t}}} < \epsilon^{\MAD}$ for both $i \in \{1, 2\}$; but this contradicts with the assumption that $\len(\mathcal{I}_{\AD_{(m)}, t}) \geq \epsilon^{\MAD}$. Thus, this case cannot happen.
	
	\textbf{Case 2. $U_{\AD_{u^{\AD}_{1, t}}} \neq U_{\AD_{u^{\AD}_{2, t}}}$ and $L_{\AD_{l^{\AD}_{1, t}}} \neq L_{\AD_{l^{\AD}_{2, t}}}$:} Let $k$ be the index associated with the empirical median absolute deviation at time $t$ according to the ranking, i.e., $k = J^{\AD}_{\kappa_2, t} \backslash J^{\AD}_{\kappa_1, t}$, we then know $u^{\AD}_{1, t} = l^{\AD}_{2, t} = k$, which leads to $\len(\mathcal{I}_{\AD_k, t}) \geq \len(\mathcal{I}_{\AD_{(m)}, t}) \geq \epsilon^{\MAD}$. Since we have $\len(\mathcal{I}_{y_{(m)}, t}) < \epsilon^{\median} = \epsilon^{\MAD}/4$, we further have $\beta_{N_{k, t}} > 3\epsilon^{\MAD}/8 >\epsilon^{\MAD}/4$ according to Lemma \ref{lm_length_confidence_bound_AD}. $\NEEDY_{k, t}^{\MAD}$ holds true as we have $k \in S^{\MAD}_{middle, t}$ for either of the three sub-cases: 
	
	(1) if $\AD_k = \AD_{(m)}$, we have $k \in S^{\MAD}_{middle, t}$ according to Remark \ref{remark_needy_MAD}; 
	
	(2) if $\AD_k > \AD_{(m)}$, we know that $\AD_k \geq \AD_{(m-1)}$. Since $\AD_{(m)} \in \mathcal{I}_{\AD_{(m)}, t} \subseteq \mathcal{I}_{\AD_k, t}$ and $\AD_k \in \mathcal{I}_{\AD_k, t}$, we then know $c^{\MAD}_1 = (\AD_{(m)} + \AD_{(m-1)})/2\in \mathcal{I}_{\AD_k, t}$, which leads to $k \in S^{\MAD}_{middle, t}$;
	
	(3) if $\AD_k < \AD_{(m)}$, similar to sub-case (2), we have $c^{\MAD}_2 \in \mathcal{I}_{k, t}$, which leads to $k \in S^{\MAD}_{middle, t}$.
\end{proof}

\begin{lemma}
	\label{lm_MAD_needy}
	Assume $\len(\mathcal{I}_{y_{(m)}, t}) < \epsilon^{\median}$. If $\len(\mathcal{I}_{\AD_{(m)}, t}) \geq \epsilon^{\MAD}$, then there exists an arm $k = l^{\AD}_{i, t}$ or $k = u^{\AD}_{i, t}$ such that $\NEEDY^{\MAD}_{k, t}$ holds.
\end{lemma}

\begin{proof}
	This Lemma is a direct consequence of Lemma \ref{lm_MAD_needy_1} and Lemma \ref{lm_MAD_needy_2}.
\end{proof}

\subsection{Sample Complexity Analysis}
\label{appendix_ROAILucb_sample_complexity}

We analyze the sample complexity upper bound of \roailucb in this section. To reduce the clutter, we will sometimes use the notation $[a \vee b] = \max\{a, b\}$, and use $\vee$ to represent \textbf{or}.

\begin{lemma}
	\label{lm_pulls_theta}
	There exists a universal constant $C$, for any $k \in [n]$, if
	\begin{equation*}
		N_{k, t} \geq \lambda^{\theta}_{k} := C \frac{ 1 }{ \left( \Delta^{\theta}_k \right)^2 } \log  \left( \frac{n}{\delta \Delta^{\theta}_{k}} \right),
	\end{equation*}
then $\NEEDY^{\theta}_{k, t}$ cannot happen.
\end{lemma}
\begin{proof}
	According to Remark \ref{remark_sample_needed}, we see there exists a universal constant $C$ such that when $N_{k, t} \geq \lambda_k^{\theta}$, we have $\beta_{N_{k, t}} < {\Delta_k^{\theta}}/{4}$. Since $\min_{i = 1, 2} \{|c^{\theta}_i - y_{k}| \} \geq {\Delta_k^{\theta}}/{2}$ by definition, we then know $c_i^{\theta} \notin \mathcal{I}_{k, t}$ when $N_{k, t} \geq \lambda_k^{\theta}$, which indicates that $\NEEDY^{\theta}_{k, t}$ cannot happen.
\end{proof}

\begin{lemma}
	\label{lm_pulls_median}
	There exists a universal constant $C$, for any $k \in [n]$, if
	\begin{equation*}
	N_{k, t} \geq \lambda^{\median}_{k} := C \frac{ 1 }{ \left[\Delta^{\median}_k \vee \epsilon^{\median} \right] ^2 } \log \left(  \frac{n}{\delta \left[ \Delta^{{\median}}_{k} \vee  \epsilon^{{\median}} \right]} \right),
	\end{equation*}
	then $\NEEDY^{{\median}}_{k, t}$ cannot happen.
\end{lemma}
\begin{proof}
	According to Remark \ref{remark_sample_needed}, we see there exist a universal constant $C$ such that when $N_{k, t} \geq \lambda_k^{\theta}$, we have $\beta_{N_{k, t}} < \max \left\{ {\Delta_k^{{\median}}}/{4}, {\epsilon^{{\median}}}/{2} \right\}$.\footnote{Note that $\max \left\{ {\Delta_k^{{\median}}}/{4}, {\epsilon^{{\median}}}/{2} \right\}$ and $\max \left\{ {\Delta_k^{{\median}}}, {\epsilon^{{\median}}} \right\}$ are in the same order.} We consider the following two cases:
	
	\textbf{Case 1. $\beta_{N_{k, t}} \leq \epsilon^{\median}/2$:} We directly know that $\NEEDY^{{\median}}_{k, t}$ cannot happen according to Definition \ref{def_ROAILucb_median}.
	
	\textbf{Case 2. $\beta_{N_{k, t}} \leq {\Delta_k^{{\median}}}/{4}$:} Since $\min_{i = 1, 2} \{|c^{{\median}}_i - y_{k}| \} \geq {\Delta_k^{{\median}}}/{2}$ by definition, we then know $c^{\median}_i \notin \mathcal{I}_{k, t}$ when $N_{k, t} \geq \lambda_k^{\median}$, which indicates that $\NEEDY^{{\median}}_{k, t}$ cannot happen.
\end{proof}

\begin{lemma}
	\label{lm_pulls_MAD}
	Assume $\len(\mathcal{I}_{y_{(m)}, t}) < \epsilon^{{\median}}= {\epsilon^{\MAD}}/{4}$. There exists a universal constant $C$, for any $k \in [n]$, if
	\begin{equation*}
		N_{k, t} \geq \lambda^{\MAD}_{k} := C \frac{ 1 }{ \left[ \Delta^{\MAD}_k \vee \epsilon^{\MAD} \right]^2 } \log  \left( \frac{n}{\delta  \left[ \Delta^{\MAD}_{k} \vee \epsilon^{\MAD} \right] } \right),
	\end{equation*}
 then $\NEEDY^{\MAD}_{k, t}$ cannot happen.
\end{lemma}

\begin{proof}
	According to Remark \ref{remark_sample_needed}, it's easy to see that there exist a universal constant $C$ such that when $N_{k, t} \geq \lambda_k^{\theta}$, we have $\beta_{N_{k, t}} < \max \left\{ {\Delta_k^{\MAD}}/{8} , {\epsilon^{\MAD}}/{4} \right\}  \leq \max \left\{ {\Delta_k^{\MAD}}/{4} - {\epsilon^{\MAD}}/{4} , {\epsilon^{\MAD}}/{4} \right\} $, where the second inequality is a mathematical fact obtained by comparing ${\Delta_k^{\MAD}}/{8}$ and ${\epsilon^{\MAD}}/{4} $ and also noticing ${\Delta_k^{\MAD}}/{4} - {\epsilon^{\MAD}}/{4} = {\Delta_k^{\MAD}}/{8} + {\Delta_k^{\MAD}}/{8} - {\epsilon^{\MAD}}/{4}$. As before, we consider the following two cases.
	
	\textbf{Case 1. $\beta_{N_{k, t}} < {\epsilon^{\MAD}}/{4} $:} We directly know that $\NEEDY^{\MAD}_{k, t}$ cannot happen according to Definition \ref{def_ROAILucb_MAD}.
	
	\textbf{Case 2. $\beta_{N_{k, t}} < {\Delta_k^{\MAD}}/{4} - {\epsilon^{\MAD}}/{4}$:} If $N_{k, t} \geq \lambda_k^{\MAD}$ and $\len(\mathcal{I}_{y_{(m)}, t}) \leq \epsilon^{{\median}} = {\epsilon^{\MAD}}/{4}$, we then have
	\begin{equation*}
	\len(\mathcal{I}_{\AD_k, t}) \leq \len(\mathcal{I}_{k, t})   +\len(\mathcal{I}_{y_{(m)}, t}) < 2\left({\Delta_k^{\MAD}}/{4} - {\epsilon^{\MAD}}/{4}\right) +  {\epsilon^{\MAD}}/{4} < {\Delta_k^{\MAD}}/{2}
	\end{equation*}
	according to Lemma \ref{lm_length_confidence_bound_AD}. Since $\min_{i = 1, 2} \{|c^{\MAD}_i - \AD_{k}| \} \geq {\Delta_k^{\MAD}}/{2}$ by definition, we then know $c^{\MAD}_i \notin \mathcal{I}_{\AD_k, t}$, which indicates that $\NEEDY^{\MAD}_{k, t}$ cannot happen.
\end{proof}

The sample complexity of \roailucb is characterized in the following theorem. 

\begin{theorem}
	\label{thm_complexity_Lcub}
	With probability of at least $1 - \delta$, the sample complexity of \roailucb is upper bounded by 
	\begin{equation*}
		O \left(  \sum_{i=1}^n  \frac{ \log \left( n / \delta \tilde\Delta_i^{*} \right)}{(\tilde\Delta_i^{*})^2} \right),
	\end{equation*}

			where	
	\begin{equation*}
	\tilde{\Delta}^*_i = \max \{\Delta_*^{\theta}/(1+k),  \min \{ \Delta_i^{\theta}, \Delta_i^{\median}, \Delta_i^{\MAD} \} \}.
	\end{equation*}
\end{theorem}

\begin{proof}
	
	We only need to upper bound the total number of rounds performed by Algorithm \ref{algorithm_ROAILucb} before termination as Algorithm \ref{algorithm_ROAILucb} only plays a constant number of arms at each round. To simplify the analysis, we first define the following notations:
	
	\begin{align*}
		\text{NT}_t &= ( \text{Algorithm \ref{algorithm_ROAILucb} doesn't terminate at round} \ t  ), \\
		A_t & = ( \len(\mathcal{I}_{y_{(m)}, t})\geq \epsilon^{{\median}}   ),\\
		B_t & = (\len(\mathcal{I}_{\AD_{(m)}, t}) \geq \epsilon^{\MAD} ),\\
		C_t & = \text{NT}_t \cap \NEEDY^{\theta}_{\theta, t}.
	\end{align*}

	The total number of rounds up to time $T$ is 

\begingroup
\allowdisplaybreaks
	\begin{align}
	\# \text{rounds} (T)  = & \sum_{t=1}^{T} \mathds{1} \left[ \text{NT}_t \right] \nonumber\\
	= & \sum_{t=1}^{T} \mathds{1} \left[ \left( \text{NT}_t \cap \neg \NEEDY^{\theta}_{\theta, t}\right) \bigcup \left( \text{NT}_t \cap \NEEDY^{\theta}_{\theta, t} \right) \right] \nonumber \\
	= & \sum_{t=1}^{T} \mathds{1} \biggl[ \left( \NEEDY^{\theta}_{l_{\theta, t}} \cup \NEEDY^{\theta}_{u_{\theta, t}}  \right) \bigcup \Bigl( C_t \cap A_t \Bigr) \bigcup \Bigl( C_t \cap \neg A_t  \cap B_t\Bigr) \bigcup  \Bigl( C_t \cap \neg A_t  \cap \neg B_t\Bigr)\biggr] \label{eq_lucb_rounds_line3}\\
	\leq & \sum_{t=1}^{T} \mathds{1} \biggl[ \left( \NEEDY^{\theta}_{l_{\theta, t}} \cup \NEEDY^{\theta}_{u_{\theta, t}}  \right) \bigcup \Bigl(A_t \Bigr) \bigcup \Bigl(  \neg A_t  \cap B_t\Bigr) \biggr] \label{eq_lucb_rounds_line4}\\
	\leq & \sum_{t=1}^{T} \mathds{1} \biggl[ \bigcup_{a \in [n]} \biggl( \Bigl( \left( a = l_{\theta, t} \vee u_{\theta, t} \right) \cap \NEEDY^{\theta}_{a, t} \Bigr) \bigcup \Bigl(\left( a = l_{i, t} \vee u_{i, t} \right) \cap \NEEDY^{{\median}}_{a, t} \Bigr)  \label{eq_lucb_rounds_line5} \\
	& \bigcup \Bigl(  \neg A_t  \cap \left( a = l^{\AD}_{i, t} \vee u^{\AD}_{i, t} \right) \cap \NEEDY^{\MAD}_{a, t}\Bigr) \biggr)\biggr]  \nonumber \\
	\leq & \sum_{t=1}^{T} \sum_{a \in [n]} \mathds{1} \biggl[ \Bigl( \left( a = l_{\theta, t} \vee u_{\theta, t} \right) \cap \NEEDY^{\theta}_{a, t} \Bigr) \bigcup \Bigl(\left( a = l_{i, t} \vee u_{i, t} \right) \cap \NEEDY^{{\median}}_{a, t} \Bigr) \nonumber   \\
	& \bigcup \Bigl(  \neg A_t  \cap \left( a = l^{\AD}_{i, t} \vee u^{\AD}_{i, t} \right) \cap \NEEDY^{\MAD}_{a, t}\Bigr) \biggr] \nonumber  \\
	\leq & \sum_{t=1}^{T} \sum_{a \in [n]} \mathds{1} \biggl[ \Bigl( \left( a = l_{\theta, t} \vee u_{\theta, t} \right) \cap N_{a,t} \leq \lambda_a^{\theta} \Bigr) \bigcup \Bigl(\left( a = l_{i, t} \vee u_{i, t} \right) \cap N_{a,t} \leq \lambda_a^{{\median}} \Bigr)   \label{eq_lucb_rounds_line9} \\
	& \bigcup \Bigl(  \neg A_t  \cap \left( a = l^{\AD}_{i, t} \vee u^{\AD}_{i, t} \right) \cap N_{a,t} \leq \lambda_a^{\MAD}\Bigr) \biggr]\nonumber  \\
	\leq &  \sum_{a \in [n]}  \sum_{t=1}^{T} \mathds{1} \biggr[ \Bigr( \left( a = l_{\theta, t} \vee u_{\theta, t} \right) \cap N_{a,t} \leq \lambda_a^{\theta} \Bigl) \bigcup \Bigl(\left( a = l_{i, t} \vee u_{i, t} \right) \cap N_{a,t} \leq \lambda_a^{{\median}} \Bigr)  \nonumber \\
	& \bigcup \Bigl(   \left( a = l^{\AD}_{i, t} \vee u^{\AD}_{i, t} \right) \cap N_{a,t} \leq \lambda_a^{\MAD}\Bigr) \biggl] \nonumber \\
	\leq & \sum_{a \in [n]} \max_{j \in \{\theta, {\median}, \MAD\}} \{ \lambda_a^j \}\nonumber .
	\end{align}
	\endgroup

	where the Eq. \eqref{eq_lucb_rounds_line3} comes form Lemma \ref{lm_non_termination_global} and Lemma \ref{lm_non_termination_local}; the Eq. \eqref{eq_lucb_rounds_line4} is derived by noticing $C_t \cap \neg A_t  \cap \neg B_t = \emptyset$; the Eq. \eqref{eq_lucb_rounds_line5} comes from Lemma \ref{lm_median_needy} and Lemma \ref{lm_MAD_needy}; and the Eq. \eqref{eq_lucb_rounds_line9} comes from Lemma \ref{lm_pulls_theta}, Lemma \ref{lm_pulls_median} and Lemma \ref{lm_pulls_MAD}.
	
	Notice the fact that $\epsilon^{{\median}} = {\epsilon^{\MAD}}/{4} = \Theta \left( {\Delta^\theta_*}/{(1+k)} \right)$, and for $\forall i \in [n]$, $\Delta^{\theta}_{i} \geq \Delta^{\theta}_*/(1+k)$. An analysis with respect to all possible orderings over $\{\Delta^{\theta}_*/(1+k), \Delta^{\theta}_{i}, \Delta^{\median}_i, \Delta^{\MAD}_i\}$ leads to the following result

	\begin{equation*}
	\sum_{a \in [n]} \max_{j \in \{\theta, {\median}, \MAD\}} \{ \lambda_a^j \} \leq O \left(  \sum_{i=1}^n  \frac{ \log \left( n / \delta \tilde\Delta_i^{*} \right)}{(\tilde\Delta_i^{*})^2} \right).
	\end{equation*}
\end{proof}

\section{Lower Bound: Proof of \cref{thm_lower_bound_worst_case}}


Our proof of lower bounds rely on the \emph{change of measure} lemma proved in \citep{kaufmann2016complexity}, which will be restated shortly for completeness. Recall that we use $D_y = (D_{y_1}, \dots, D_{y_n})$ to represent a bandit instance, and assume each arm follows the distribution $D_{y_i} := \mathcal{N}(y_i, 1)$. $\mathbb{E}_{y} (\cdot)$ is used to represent the expectation with respect to the bandit instance $D_y$ and randomness in the algorithm. We will use ${\KL}(P, Q)$ to denote the KL-divergence between distribution $P$ and $Q$. Based on the calculation of KL-divergence between two Gaussian distributions, we also have
\begin{equation*}
 {\KL} (D_{y_1}, D_{y_2}) = {\KL} (\mathcal{N}(y_1, 1), \mathcal{N}(y_2, 1)) = {(y_1 - y_2)^2}/{2}.
\end{equation*}

\begin{lemma} (\cite{kaufmann2016complexity})
	\label{lm_change_measure}
	Let $D_y$ and $D_{y^{\prime}}$ be two bandit instances with $n$ arms such that for all $i$, the distributions $D_{y_i}$ and $D_{y_i^{\prime}}$ are mutually absolutely continuous. Let $\tau$ be a stopping time with respect to filtration $\{\mathcal{F}_t \}$. For any event $A \in \mathcal{F}_{\tau}$, we have
	\begin{equation*}
	\sum_{i = 1}^{n} \mathbb{E}_y[N_{i, \tau}] \cdot {\KL}(D_{y_i} , D_{y_i^{\prime}}) \geq d \left( \mathbb{P}_y(A), \mathbb{P}_{y^{\prime}}(A) \right),
	\end{equation*}
	where $d(x,y) = x \log \left( \frac{x}{y} \right) + (1 - x) \log \left( \frac{1-x}{1-y} \right) $, with the convention that $d(0, 0) = d(1, 1) = 0$.
\end{lemma}

For any $D_y \in {\mathcal{M}}_{n, \rho}$, we assume $y_i \geq y_{i+1}$ and use $S_o = \{1, \dots, n_1\}$ to denote the subset of outlier arms. We use $\theta$ as derived from Eq. \eqref{eq_robust_threshold}. Before get into our proof, we first define some mapping functions for each arm as followings:
\begin{equation*}
{\psi}^{\theta}_{ \rho} (y_i) = \begin{cases}
\theta - \rho & \textrm{if} \ \ y_i > \theta\\
\theta+ \rho & \textrm{if} \ \ y_i < \theta\\
\end{cases},
\end{equation*}
\begin{equation*}
{\psi}^{{\median}}_{\rho} (y_i)= \begin{cases}
y_{(m)} - \rho & \textrm{if} \ \ y_i \geq y_{(m)}\\
y_{(m)} + \rho & \textrm{if} \ \ y_i < y_{(m)}\\
\end{cases},
\end{equation*}
and 
\begin{equation*}
{\psi}^{\MAD}_{\rho} (y_i) = \begin{cases}
y_{(m)} - \rho & \textrm{if} \ \ y_i = y_{(m)} \\
y_i - \AD_i + \AD_{(m)} - \rho & \textrm{if} \ \ y_i   > y_{(m)} \ \text{and} \ \AD_i \geq \AD_{(m)}\\
y_i  -  \AD_i + \AD_{(m)}  + \rho & \textrm{if} \ \ y_i > y_{(m)} \ \text{and} \ \AD_i < \AD_{(m)}\\
y_i + \AD_i - \AD_{(m)} + \rho& \textrm{if} \ \ y_i < y_{(m)} \ \text{and} \ \AD_i \geq \AD_{(m)}\\
y_i +\AD_i - \AD_{(m)} - \rho  & \textrm{if} \ \ y_i < y_{(m)} \ \text{and} \  \AD_i < \AD_{(m)}\\
\end{cases}.
\end{equation*}
We could see that these mappings essentially map value $y_i$ to another value with deviations closely related to $\Delta^{\theta}_{i, \rho}$, $\Delta^{{\median}}_{i, \rho}$ or $\Delta^{\MAD}_{i, \rho}$, where we define
	\begin{align*}
{\Delta}^{\theta}_{i, \rho} &= \Delta^{\theta}_{i} + \rho, \\
{\Delta}^{\median}_{i, \rho} &= \Delta^{\median}_{i} + \rho, \\
{\Delta}^{\MAD}_{i, \rho} &= \Delta^{\MAD}_i + \rho. 
\end{align*}
 Next lemma shows how changing the mean of one single arm changes the output decision of the outlier arm identification problem. 

\begin{lemma}
	\label{lm_change_outlier_set}
	Suppose $D_y = (D_{y_1}, D_{y_2}, \dots, D_{y_n}) \in {\mathcal{M}}_{n, \rho} $, then for any $i \in [n]$ and any $a \in \{\theta, {\median}, \MAD\}$, the subset of outlier arms in the following instance
	\begin{equation*}
	D_{y^\prime} = (D_{y_1}, \dots, D_{y_{i-1}}, D_{\psi^a_{\rho}(y_i)}, D_{y_{i+1}}, \dots, D_{y_n})
	\end{equation*}
	is not $S_o = [n_1]$ anymore.
\end{lemma}

\begin{proof}
	
	We prove this lemma by showing that it holds for all three cases; and we use $\theta^\prime$, $y^\prime_{(m)}$ and $\AD_{(m)}^\prime$ to represent, respectively, the outlier threshold, the \text{median} and the \text{median} absolute deviation with respect to (expected rewards of) bandit instance $D_{y^\prime}$. 
	
	\textbf{Case 1: } $a = \theta$.
	If $|\theta^\prime - \theta| < \rho$, then arm $i$ is removed or added to $S_o$ in instance $D_{y^\prime}$; otherwise, at least one of $\{u_1, u_2, l_1, l_2\}$ is removed or added to $S_o$.
	
	\textbf{Case 2: } $a = {\median}$.
	According to definition of $\eta$ and the fact $\rho < \eta$, arm $i$ becomes the unique \text{median} arm after mapping $y_i$ to $\psi^{\median}_{\rho}(y_i)$; we thus have $|y^\prime_{(m)} - y_{(m)}| = \rho$ and $\AD_i^\prime = 0$. Furthermore, we have $|\AD^\prime_j - \AD_j| = \rho$ for all $j \in [n]\backslash{i}$ as the median value is changed by $\rho$ and $\min \{ y_{(m)} - y_{(m+1)}, y_{(m-1)} - y_{(m)} \}\geq 2 \eta > 2\rho$ by Definition \ref{def_minimax_lower_bound}.

	If $\AD_i < \AD_{(m)}$, since $\min \{ \AD_{(m)} - \AD_{(m+1)}, \AD_{(m-1)} - \AD_{(m)} \}\geq 2 \eta > 2\rho$ by Definition \ref{def_minimax_lower_bound}, we know that the arm associated with the MAD value in $D_y$ is still associated with the MAD value in $D_{y^\prime}$. We thus have $|\AD_{(m)} - \AD^\prime_{(m)}| = \rho$. Since $k \geq 2$ by definition, we have $|\theta^\prime - \theta| \geq \rho$, resulting in at least one of $\{u_1, u_2, l_1, l_2\}$ being removed or added to $S_o$.
	
	If $\AD_i \geq \AD_{(m)}$, we know there exists an arm $p$ such that $\AD_{p} \leq \AD_{(m+1)}$ is now associated with $\AD_{(m)}^\prime$ in $D_{y^\prime}$ (since $\AD_i^\prime$ becomes $0$ in $D_{y^\prime}$). Since $\min \{ \AD_{(m)} - \AD_{(m+1)} \}\geq 2 \eta> 2\rho$ by Definition \ref{def_minimax_lower_bound}, we thus have $|\AD_{(m)} - \AD^\prime_{(m)}| \geq \rho$ as $\AD^\prime_{(m)}=\AD^\prime_p \leq \AD_p + \rho \leq \AD_{(m+1)} + \rho$. We then have $|\theta^\prime - \theta| \geq \rho$ as $k \geq 2$ by definition; this further results in at least one of $\{u_1, u_2, l_1, l_2\}$ being removed or added to $S_o$.

	\textbf{Case 3: } $a = \MAD$.
	If $y_i = y_{(m)}$, same analysis appears in Case 2 applies here and leads to at least one of $\{u_1, u_2, l_1, l_2\}$ will be removed or added to $S_o$.
	
	If $y_i \neq y_{(m)}$, we have $y^\prime_{(m)} = y_{(m)}$ (notice that $\AD_{(m)} \geq 2 \eta > 2\rho$ by assumption) and $|\AD^\prime_{(m)} - \AD_{(m)}| = \rho$ by the construction of ${\psi}^{\MAD}_{\rho} (y_i)$, which results in $|\theta^\prime - \theta| \geq \rho$ and thus at least one of $\{u_1, u_2, l_1, l_2\}$ being removed or added to $S_o$.
\end{proof}

Now we restated Theorem \ref{thm_lower_bound_worst_case} and provide the proof.

\lowerBound*

\begin{proof}
	For any $i \in [n]$ and any $a \in \{\theta, {\median}, \MAD\}$, we construct 
	\begin{equation*}
		D_{y^\prime} = (D_{y_1}, \dots, D_{y_{i-1}}, D_{\psi^a_{\rho}(y_i)}, D_{y_{i+1}}, \dots, D_{y_n}).
	\end{equation*}
	We also define the event $A = \{ \hat{S}_o = [n_1] \}$, which is measurable with respect to $\mathcal{F}_{\tau}$. For any $\delta$-PAC algorithm, according to its definition and Lemma \ref{lm_change_outlier_set}, we have $\mathbb{P}_y(A)\geq 1 - \delta$ and $\mathbb{P}_{y^{\prime}}(A) \leq \delta$. Thus, according to Lemma \ref{lm_change_measure}, we have
	
	\begin{equation}
	\label{eq:lower_bound_1}
	\mathbb{E}_y[N_{i, \tau}] \cdot {\KL}(D_{y_i}, D_{y_i^{\prime}}) \geq d(1- \delta, \delta) \geq \log \left( \frac{1}{2.4 \delta}  \right), 
	\end{equation}
	where we use the property that for $x \in [0,1]$, $d(x, 1-x) \geq \log \left(\frac{1}{2.4 \delta}\right)$ for the last inequality. \cref{eq:lower_bound_1} further gives us 
	\begin{equation}
		\label{eq:lower_bound_2}
	\mathbb{E}_y[N_{i, \tau}] \geq \frac{1}{ {\KL}\left(D_{y_i}, D_{\psi^a_\rho(y_i)} \right)  } \log \left(\frac{1}{2.4 \delta}\right).
	\end{equation}
	Combining \cref{eq:lower_bound_2} with ${\KL}(D_{y_i}, D_{\psi^a_\rho(y_i)} )  = {2}/{\left( \Delta^a_{i, \rho} \right)^2}$ for $a \in \{\theta, \median, \MAD\}$ and $\Delta^{*}_{i, \rho} = \min \{ {\Delta}^{\theta}_{i, \rho}, {\Delta}^{\median}_{i, \rho} , {\Delta}^{\MAD}_{i, \rho} \} \geq \rho$, we have
	\begin{equation}
	\label{eq:lower_bound_arm_i}
	\mathbb{E}_y[N_{i, \tau}] \geq \frac{2}{\left({\Delta}_{i, \rho}^{*}\right)^2} \log\left(\frac{1}{2.4 \delta}\right)
.	\end{equation}
	
	For any arm $i$ such that $\min \{ \Delta_i^{\theta}, \Delta_i^{\median}, \Delta_i^{\MAD} \} \neq 0 $, we have $\min \{ \Delta_i^{\theta}, \Delta_i^{\median}, \Delta_i^{\MAD} \} \geq \rho/2$ according to the construction of ${\cal M}_{n, \rho}$. Thus, ${2}/{\left({\Delta}_{i, \rho}^{*}\right)^2} = {2}/{\left( \min \{ \Delta_i^{\theta}, \Delta_i^{\median}, \Delta_i^{\MAD} \}+ \rho \right)^2} \geq {2}/{\left( 3 \min \{ \Delta_i^{\theta}, \Delta_i^{\median}, \Delta_i^{\MAD} \} \right)^2}  \geq {2}/{\left( 3 \Delta^*_i \right)^2} $.
	
	For any arm $i$ such that $\min \{ \Delta_i^{\theta}, \Delta_i^{\median}, \Delta_i^{\MAD} \} = 0$, we have $\Delta^{\theta}_* \geq \rho/2$ according to the construction of ${\cal M}_{n, \rho}$. Thus, ${2}/{\left({\Delta}_{i, \rho}^{*}\right)^2} = 2/ \rho^2 \geq 2 / {\left( 2 \Delta^{\theta}_* \right)^2} = 2 / {\left( 2 \Delta^*_i \right)^2}$.

	Combining \cref{eq:lower_bound_arm_i} with the above analysis, we have 
		\begin{equation*}
	\mathbb{E}_y[N_{i, \tau}] \geq \frac{1}{5 \left({\Delta}_{i}^{*}\right)^2} \log\left(\frac{1}{2.4 \delta}\right).
	\end{equation*}
	Summing over all $i \in [n]$ yields the desired bound in \cref{thm_lower_bound_worst_case}.
\end{proof}

\section{Heuristic to Reduce Sample Complexity: Proof of \cref{thm_subsampling}}
\label{appendix_heuristic}

\subSampling*

\begin{proof}
The subsampling algorithm is implemented as in Algorithm \ref{algorithm_ROAILucb}, with some notational changes to adapt to the subset $\Omega$, as described here. We still assume the total number of arms is $n$, but set $|\Omega| = 2m-1$.\footnote{Recall we simply choose the median as $m$ if $|\Omega| = 2m$.} $y_{(m)}$, $\AD_i$, $\AD_{(m)}$ and $\theta$ are all calculated with respect to arms in $\Omega$. We use the notation $J_{\kappa_i, t}$ to denote $\kappa_i$ arms \emph{in $\Omega$} with the largest empirical means $\{\hat{y}_{i, t}\}$, and $J^{\AD}_{\kappa_i, t}$ to denote the $\kappa_i$ arms \emph{in $\Omega$} with the largest empirical absolute deviations $\{\widehat{\AD}_{i, t}\}$. Since we are mainly interested in shrinking confidence intervals around the median, we set $\kappa_1 = m-1$ and $\kappa_2 = m$. 

The $\NEEDY^{\theta}_{i, t}$ event remains the same for all arms in $[n]$; however, we will have $\NEEDY^{\median}_{i, t}$ and $\NEEDY^{\MAD}_{i, t}$ events only for arms in $\Omega$ as arms outside $\Omega$ are not involved in the construction of the outlier threshold. Eq. \eqref{eq_ROAILucb_constant} and lemmas in \cref{appendix_ROAILucb_median} and \cref{appendix_ROAILucb_MAD} can be adapted to the subset $\Omega$. Since Algorithm \ref{algorithm_ROAILucb} pulls a constant number of arms each round, we only need to upper bound the total number of rounds up to time $T$. Similar to the analysis in the proof of \cref{thm_complexity_Lcub}, we have 

\begingroup
\allowdisplaybreaks
\begin{align*}
\# \text{rounds} (T)  = & \sum_{t=1}^{T} \mathds{1} \left[ \text{NT}_t \right] \\
= & \sum_{t=1}^{T} \mathds{1} \left[ \left( \text{NT}_t \cap \neg \NEEDY^{\theta}_{\theta, t}\right) \cup \left( \text{NT}_t \cap \NEEDY^{\theta}_{\theta, t} \right) \right] \\
= & \sum_{t=1}^{T} \mathds{1} \biggl[ \left( \NEEDY^{\theta}_{l_{\theta, t}} \cup \NEEDY^{\theta}_{u_{\theta, t}}  \right) \bigcup \Bigl( C_t \cap A_t \Bigr) \bigcup \Bigl( C_t \cap \neg A_t  \cap B_t\Bigr) \bigcup  \Bigl( C_t \cap \neg A_t  \cap \neg B_t\Bigr)\biggr] \\
\leq & \sum_{t=1}^{T} \mathds{1} \biggl[ \left( \NEEDY^{\theta}_{l_{\theta, t}} \cup \NEEDY^{\theta}_{u_{\theta, t}}  \right) \bigcup \Bigl(A_t \Bigr) \bigcup \Bigl(  \neg A_t  \cap B_t\Bigr) \biggr] \\
\leq & \sum_{t=1}^{T} \mathds{1} \biggl[ \bigcup_{a \in \Omega} \biggl( \Bigl( \left( a = l_{\theta, t} \vee u_{\theta, t} \right) \cap \NEEDY^{\theta}_{a, t} \Bigr) \bigcup \Bigl(\left( a = l_{i, t} \vee u_{i, t} \right) \cap \NEEDY^{{\median}}_{a, t} \Bigr)  \\
& \bigcup \Bigl(  \neg A_t  \cap \left( a = l^{\AD}_{i, t} \vee u^{\AD}_{i, t} \right) \cap \NEEDY^{\MAD}_{a, t}\Bigr) \biggr)   \bigcup  \biggl( \bigcup_{i \notin \Omega} \Bigl( \left( a = l_{\theta, t} \vee u_{\theta, t} \right) \cap \NEEDY^{\theta}_{a, t} \Bigr)\biggr)\biggr] \\
\leq & \sum_{t=1}^{T} \sum_{a \in \Omega} \mathds{1} \biggl[ \Bigl( \left( a = l_{\theta, t} \vee u_{\theta, t} \right) \cap \NEEDY^{\theta}_{a, t} \Bigr) \bigcup \Bigl(\left( a = l_{i, t} \vee u_{i, t}  \right) \cap \NEEDY^{{\median}}_{a, t} \Bigr)  \\
& \bigcup \Bigl(  \neg A_t  \cap \left( a = l^{\AD}_{i, t} \vee u^{\AD}_{i, t} \right) \cap \NEEDY^{\MAD}_{a, t}\Bigr) \biggr] + \sum_{t=1}^{T} \sum_{a \notin \Omega} \mathds{1} \biggl[ \Bigl( \left( a = l_{\theta, t} \vee u_{\theta, t} \right) \cap \NEEDY^{\theta}_{a, t} \Bigr) \biggr]  \\
\leq & \sum_{t=1}^{T} \sum_{a \in \Omega} \mathds{1} \biggl[ \Bigl( \left( a = l_{\theta, t} \vee u_{\theta, t} \right) \cap N_{a,t} \leq \lambda_a^{\theta} \Bigr) \bigcup \Bigl(\left( a = l_{i, t} \vee u_{i, t}  \right) \cap N_{a,t} \leq \lambda_a^{{\median}} \Bigr)   \\
& \bigcup \Bigl(  \neg A_t  \cap \left( a = l^{\AD}_{i, t} \vee u^{\AD}_{i, t} \right) \cap N_{a,t} \leq \lambda_a^{\MAD}\Bigr) \biggr] + \sum_{t=1}^{T} \sum_{a \notin \Omega} \mathds{1} \biggl[ \Bigl( \left( a = l_{\theta, t} \vee u_{\theta, t} \right) \cap N_{a,t} \leq \lambda_a^{\theta} \Bigr) \biggr]  \\
\leq &  \sum_{a \in \Omega}  \sum_{t=1}^{T} \mathds{1} \biggr[ \Bigr( \left( a = l_{\theta, t} \vee u_{\theta, t} \right) \cap N_{a,t} \leq \lambda_a^{\theta} \Bigl) \bigcup \Bigl(\left( a = l_{i, t} \vee u_{i, t}  \right) \cap N_{a,t} \leq \lambda_a^{{\median}} \Bigr)  \\
& \bigcup \Bigl(   \left( a = l^{\AD}_{i, t} \vee u^{\AD}_{i, t} \right) \cap N_{a,t} \leq \lambda_a^{\MAD}\Bigr) \biggl] + \sum_{a \notin \Omega}  \sum_{t=1}^{T} \mathds{1} \biggl[ \Bigl( \left( a = l_{\theta, t} \vee u_{\theta, t} \right) \cap N_{a,t} \leq \lambda_a^{\theta} \Bigr) \biggr] \\
\leq & \sum_{a \in \Omega} \max_{j \in \{\theta, {\median}, \MAD\}} \{ \lambda_a^j \} + \sum_{a \notin \Omega}   \lambda_a^\theta .
\end{align*}
\endgroup
The rest of the proof is the similar to that in \cref{thm_complexity_Lcub}.
\end{proof}

\end{document}